\documentclass{article} 
\usepackage{iclr,times}

\usepackage{amssymb}
\usepackage{amsmath,amsfonts,bm}

\newcommand{\gap}{\,\,\,\,\,\,\,\,}  
\mathchardef\mhyphen="2D

\newcommand\norm[1]{\left\lVert#1\right\rVert}

\def\1{\bm{1}}

\newcommand{\real}{\mathbb{R}}

\definecolor{titlepagecolor}{cmyk}{75,68,67,90}
\definecolor{titlepagecolor2}{rgb}{1.0, 0.08, 0.58}
\definecolor{emerald}{rgb}{0.31, 0.78, 0.47}
\definecolor{deeppink}{HTML}{D14064}
\definecolor{lowpink}{HTML}{ffe6ec}
 
\definecolor{lowblue}{HTML}{E1EBFE}

\let\oldforall\forall
\renewcommand{\forall}{\oldforall\, }

\usepackage{color}   

\definecolor{mylightbluetitle}{RGB}{60,113,183}
\definecolor{structurecolorblue}{RGB}{60,113,183}
\definecolor{structurecolorgreen}{RGB}{63,145,182}
\colorlet{structurecolor}{structurecolorblue}
\definecolor{structurecolorelegant}{RGB}{60,113,183}
\definecolor{structurecolorlt}{RGB}{31,119,185}

\definecolor{structurecolorHighTheoremBlue}{RGB}{220,227,248}
\definecolor{structurecolorHighTheoremGreen}{RGB}{188,222,231}
\colorlet{structurecolorHighTheorem}{structurecolorHighTheoremBlue}

\newcommand{\mdframecolorBox}{gray!15}  

\definecolor{winestain}{rgb}{0.5,0,0}
\definecolor{mydarkblue}{rgb}{0,0.08,0.45}
\definecolor{mydarkred}{rgb}{0.70,0.00,0.00}
\definecolor{mydarkgreen}{rgb}{0.00,0.30,0.00}
\definecolor{mydarkyellow}{RGB}{197,151,13}
\definecolor{mydarkpurple}{RGB}{149,18,192}
\definecolor{mydarkgray}{RGB}{64,64,64}

\definecolor{color0}  {RGB}{174,225,254} 
\definecolor{color1}  {RGB}{220,227,248} 
\definecolor{color2}  {RGB}{28,130,185} 
\definecolor{color3}  {RGB}{255,253,250} 
\definecolor{colormiddleright}  {RGB}{245,253,250} 
\definecolor{colorbottomleft}  {RGB}{255,243,250} 
\definecolor{coloruppermiddle}  {RGB}{255,253,230} 
\definecolor{colormiddleleft}  {RGB}{255,244,237}
\definecolor{colorcr}  {RGB}{249,253,232} 
\definecolor{colorreduction}  {RGB}{255,235,254} 
\definecolor{colorqr}  {RGB}{254,221,199} 
\definecolor{colorbiconjugate}  {RGB}{251,149,161} 
\definecolor{colorsvd}  {RGB}{215,247,235} 
\definecolor{colorupperright}  {RGB}{239,246,251} 
\definecolor{colorspectral}  {RGB}{206,226,243} 
\definecolor{colorbottomright}  {RGB}{220,224,236} 
\definecolor{coloreigenvalue}  {RGB}{197,203,224} 
\definecolor{colorcp} {RGB}{217, 234, 186} 
\definecolor{colorcpborder} {RGB}{233, 243, 216} 
\definecolor{colorupperleft}  {RGB}{235,243,240} 
\definecolor{colorsemidefinite}  {RGB}{217,232,226} 
\definecolor{colormiddle} {RGB}{235, 240,255}
\definecolor{colorlu}  {RGB}{220,227,255} 
\definecolor{colorals}  {RGB}{240,230,255} 
\definecolor{coloralsbkg}  {RGB}{248,243,255} 
\definecolor{canaryyellow}{rgb}{1.0, 0.75, 0.0}
\definecolor{bluepigment}{rgb}{0.0, 0.0, 1.0}
\definecolor{canarypurple}{RGB}{208, 13, 241}
\definecolor{colorGreenOcre}{RGB}{51,102,0} 
\definecolor{colorBlue2}{RGB}{200,207,248}
\definecolor{shadecolor}{gray}{0.75}

\newcommand{\hadaprod}{\circ}

\newcommand{\cspace}{\mathcal{C}}
\newcommand{\nspace}{\mathcal{N}}
\newcommand{\bzero}{\mathbf{0}}

\newcommand{\diag}{\mathrm{diag}}
\newcommand{\rank}{\mathrm{rank}}

\newcommand{\trace}{\mathrm{tr}}

\newcommand{\bSigma}{\boldsymbol\Sigma}

\newcommand{\widetildebA}{\widetilde{\bm{A}}}

\newcommand{\widetildebM}{\widetilde{\bm{M}}}

\newcommand{\widetildebW}{\widetilde{\bm{W}}}

\newcommand{\widetildebZ}{\widetilde{\bm{Z}}}
\newcommand{\widetildeba}{\widetilde{\bm{a}}}
\newcommand{\widetildebb}{\widetilde{\bm{b}}}
\newcommand{\widetildebc}{\widetilde{\bm{c}}}

\newcommand{\widetildebw}{\widetilde{\bm{w}}}

\newcommand{\widetildebz}{\widetilde{\bm{z}}}

\newcommand{\ba}{\bm{a}}
\newcommand{\bA}{\bm{A}}
\newcommand{\bb}{\bm{b}}
\newcommand{\bB}{\bm{B}}
\newcommand{\bc}{\bm{c}}
\newcommand{\bC}{\bm{C}}  
\newcommand{\bd}{\bm{d}}
\newcommand{\bD}{\bm{D}}
\newcommand{\be}{\bm{e}}
\newcommand{\bE}{\bm{E}}

\newcommand{\bI}{\bm{I}}

\newcommand{\bK}{\bm{K}}

\newcommand{\bM}{\bm{M}}

\newcommand{\bo}{\bm{o}}

\newcommand{\bp}{\bm{p}}
\newcommand{\bP}{\bm{P}}

\newcommand{\bQ}{\bm{Q}}

\newcommand{\bS}{\bm{S}}

\newcommand{\bT}{\bm{T}}

\newcommand{\bU}{\bm{U}}

\newcommand{\bV}{\bm{V}}
\newcommand{\bw}{\bm{w}}
\newcommand{\bW}{\bm{W}}
\newcommand{\bx}{\bm{x}}
\newcommand{\bX}{\bm{X}}
\newcommand{\by}{\bm{y}}

\newcommand{\bz}{\bm{z}}
\newcommand{\bZ}{\bm{Z}}

\DeclareMathAlphabet{\mathsfit}{\encodingdefault}{\sfdefault}{m}{sl}
\SetMathAlphabet{\mathsfit}{bold}{\encodingdefault}{\sfdefault}{bx}{n}

\def\sS{{\mathbb{S}}}

\usepackage{flushend} 
\usepackage{balance}
 
\usepackage{array,multirow,graphicx}
\usepackage{changes}
\usepackage{booktabs}
\usepackage{graphicx} 
\usepackage{xparse}
\usepackage{hyperref}
\usepackage{url}
\usepackage{subfigure} 
\usepackage{cancel}

\usepackage{sidecap}
\sidecaptionvpos{figure}{t}
\usepackage{verbatimbox}
\usepackage[labelfont=bf]{caption} 
\usepackage{wrapfig}

\usepackage{algorithm}
\usepackage{algpseudocode}
\usepackage{graphics}
\usepackage{epsfig}

\usepackage{color}   
\definecolor{winestain}{rgb}{0.5,0,0}
\definecolor{ocre}{RGB}{51,102,0} 
\definecolor{colorBlue2}{RGB}{200,207,248}
\definecolor{mydarkblue}{rgb}{0,0.08,0.45}
\definecolor{mylightbluetext}{rgb}{0,0.08,0.45}
\usepackage{hyperref}
\hypersetup{
	colorlinks=true, 
	linktoc=all,     
	linkcolor=mydarkblue,  
	anchorcolor=blue,
	citecolor=mydarkblue,
}
\PassOptionsToPackage{numbers, compress}{natbib}
\usepackage[numbers]{natbib}

\renewenvironment{thebibliography}[1]{
	\begin{oldthebibliography}{#1}
		\setlength{\itemsep}{0em}
		\setlength{\parskip}{0.07em}
	}
	{
	\end{oldthebibliography}
}

\title{Low-Rank Approximation, Adaptation, and Other Tales}

\author{
Jun Lu 
\\
\texttt{jun.lu.locky@gmail.com} \\
}

\newtheorem{theorem}{Theorem}
\newtheorem{definition}{Definition}
\newtheorem{lemma}{Lemma}
\newtheorem{remark}{Remark}
\newtheorem{proof}{Proof}
\newcommand{\BlackBox}{\rule{1.5ex}{1.5ex}} 
\renewenvironment{proof}{\par\noindent{\bf Proof\ }}{\hfill\BlackBox\\[2mm]}

\begin{document}

\maketitle

\begin{abstract}
Low-rank approximation is a fundamental technique in modern data analysis, widely utilized across various fields such as signal processing, machine learning, and natural language processing. Despite its ubiquity, the mechanics of low-rank approximation and its application in adaptation can sometimes be obscure, leaving practitioners and researchers with questions about its true capabilities and limitations. This paper seeks to clarify low-rank approximation and adaptation by offering a comprehensive guide that reveals their inner workings and explains their utility in a clear and accessible way.
Our focus here is to develop a solid intuition for how low-rank approximation and adaptation operate, and why they are so effective. We begin with basic concepts and gradually build up to the mathematical underpinnings, ensuring that readers of all backgrounds can gain a deeper understanding of low-rank approximation and adaptation. We strive to strike a balance between informal explanations and rigorous mathematics, ensuring that both newcomers and experienced experts can benefit from this survey.
Additionally, we introduce new low-rank decomposition and adaptation algorithms that have not yet been explored in the field, hoping that future researchers will investigate their potential applicability.

\end{abstract}

\section{Introduction}

The rapid advancement in sensor technology and computer hardware has led to a significant increase in the volume of data being generated, presenting new challenges for data analysis. This vast amount of data often contains noise and other distortions, necessitating preprocessing steps to make it suitable for scientific analysis. For instance, signals received by antenna arrays are frequently contaminated by noise and other degradations.
On the other hand, although we would like to analyze the data at higher frequencies, the computational requirements are not always met.
To effectively analyze the data, it is necessary to reconstruct or represent it in a way that reduces inaccuracies while maintaining certain feasibility conditions.

Additionally, in many scenarios, the data collected from complex systems is the result of numerous interrelated variables working together. When these variables are not clearly defined, the information in the original data can be ambiguous and overlapping. By creating a simplified model that captures the essential dynamics of the system, we can achieve a level of accuracy that closely matches the original system. A common approach to simplifying the data involves replacing the original data with a lower-dimensional representation obtained through subspace approximation, which helps in removing noise, reducing the complexity of the model, and ensuring feasibility.

Low-rank approximations or low-rank matrix decompositions are critical in a wide range of applications because they enable us to represent a given matrix as the product of two or more matrices with fewer dimensions. These techniques help capture the fundamental structure of a matrix while filtering out noise and redundancies. Common methods for low-rank matrix decomposition include singular value decomposition (SVD), principal component analysis (PCA),  multiplicative update nonnegative matrix factorization (NMF), and the alternating least squares (ALS) approach.

An illustrative example of the application of low-rank matrix decomposition is found in collaborative filtering tasks, such as the Netflix Prize competition \citep{bennett2007netflix}, where the objective is to predict user ratings for movies based on their existing ratings for other movies. Let's denote the ratings of the  $n$-th user for the $m$-th movie by $a_{mn}$. We can define an $M\times N$ rating matrix $\bA$ (also known as the \textit{preference matrix}) with columns  $\ba_n$ ($n\in\{1,2,\ldots, N\}$) representing the ratings provided by the $n$-th user. Many of the ratings $\{a_{mn}\}$ are missing, and our goal is to accurately predict these missing ratings, i.e., to complete the matrix (known as the \textit{matrix completion} problem). Without assuming some underlying structure in the matrix, there would be no relationship between the observed and unobserved entries, making the prediction task impossible \citep{jain2017non, lu2022flexible}. 
Therefore, it is crucial to impose some structure on the matrix. A common structural assumption is that the matrix is of low rank, which makes the problem well-posed and allows for a unique solution.
Consequently, the unobserved entries can no longer be independent of the observed values. \footnote{However, this is actually a strong assumption. Consider the matrix $\bA=\sum_{i=1}^{r} \be_i\widetilde{\be}_j^\top$, where $\be_i$ is the standard basis for $\real^M$ and $\widetilde{\be}_j$ is the standard basis for $\real^N$. $\bA$ is a rank-$r$ matrix and contains only $r$ nonzero entries.  
In recommendation settings, we typically observe only a few random entries of the matrix. As a result, there is a high possibility  that some entries may never be observed. This poses a significant challenge for the matrix completion problem.
However, we will not cover this topic in this paper.}

The low-rank matrix completion problem can be formulated as finding the matrix $\widetilde{\bA}$ that minimizes the squared error between the observed entries of  $\bA$ and $\widetilde{\bA}$, subject to the constraint that $\widetilde{\bA}$ is of rank $K$. 
Consider the mask matrix $\bM\in \{0,1\}^{M\times N}$, where $m_{mn}\in \{0,1\}$~\footnote{We denote the $(m,n)$-th component of matrix $\bB$ by $b_{mn}$.} means if the user $n$ has rated the movie $m$ or not.
Then the low-rank matrix completion problem can be formulated as 
$$
\widetilde{\bA} = \mathop{\arg\min}_{\bX\in\real^{M\times N}} \sum_{m,n=0}^{M,N} (x_{mn} - a_{mn})^2\cdot m_{mn}\gap \text{s.t.} \gap \rank(\bX)\leq K.
$$
This problem is NP-hard \citep{hardt2014computational} but can be reformulated in an unconstrained form by factoring $\widetilde{\bA}$ into two matrices $\bW$ and $\bZ$ of appropriate dimensions:
$$
\widetilde{\bA} = \mathop{\arg\min}_{\substack{\bW\in\real^{M\times K}\\ \bZ\in\real^{K\times N}}} \sum_{m,n=0}^{M,N} ((\bW\bZ)_{mn}- a_{mn})^2\cdot m_{mn},
$$
which can then be optimized indirectly using an alternating algorithm.

On the other hand, low-rank approximation has gained significant attention recently due to its application in fine-tuning pre-trained large language models (LLMs). Large language models have become a cornerstone of modern natural language processing (NLP) and have driven substantial advancements in the field. The use of neural networks for language modeling dates back to the early 2000s, but the rise of LLMs began with the advent of deep learning and the introduction of the transformer architecture.
Prior to this, recurrent neural networks (RNNs) and long short-term memory (LSTM) networks were the dominant architectures for sequence modeling.

The transformer architecture, introduced in 2017 by  \citet{vaswani2017attention}, represented a significant shift away from RNNs. Transformers employ self-attention mechanisms to process input sequences in parallel, which makes them highly efficient for large-scale training. This architecture quickly became the basis for many state-of-the-art models.

The pre-training paradigm consists of initially training a model on a vast text corpus to learn general linguistic patterns, which is then fine-tuned for specific tasks (called downstream tasks). This approach gained widespread attention with the introduction of bidirectional encoder representations from transformers (BERT) in 2018 \citep{devlin2018bert}, setting new benchmarks in various NLP tasks. BERT and similar models leverage the transformer architecture to obtain contextual representations of words.

Building on the success of BERT, researchers started scaling up the size of pre-trained models. GPT-2 (Generative pre-trained transformer 2) and RoBERTa (robustly optimized BERT pre-training approach) were among the first models that demonstrated how scaling up model size could enhance performance. GPT-3, released in 2020, pushed this trend to new heights with its 175 billion parameters \citep{brown2020language}, showcasing impressive capabilities in \textit{few-shot learning}, where the model can handle tasks with minimal additional training.

Recent advancements in large language models have emphasized enhancing efficiency and scalability. As models expanded, the computational and storage demands for fine-tuning became increasingly costly. This prompted the creation of parameter-efficient tuning (PET) techniques, designed to adapt pre-trained models to new tasks with minimal alterations to their parameters. Methods such as LoRA (low-rank adaptation) and adapters are notable examples of PET approaches \citep{hu2021lora, houlsby2019parameter}.

Adapters are small neural networks placed between the layers of a pre-trained model, fine-tuned for specific tasks while keeping the rest of the model unchanged/frozen. They have proven effective in transferring knowledge from pre-training to downstream tasks.

LoRA  is a technique that employs low-rank approximation during the fine-tuning of LLMs. It keeps the pre-trained model weights frozen and injects trainable rank decomposition matrices into each layer of the transformer architecture. This approach significantly reduces the number of trainable parameters and the GPU memory required, making it more efficient than full fine-tuning. Moreover, LoRA does not introduce additional latency during inference. Once fine-tuning is complete, the trainable matrices can be merged with the frozen weights, eliminating the need for extra computations during inference.

The rest of this paper is organized as follows. In Section~\ref{sec:als} we present the low-rank decomposition model and give a more formal 
description of alternating least squares algorithm for obtaining low-rank approximations.
In Section~\ref{sec:special} we describe special matrix products; their properties are considered. 
The low-rank Hadamard, Kronecker, and the variant Khatri-Rao decompositions are consider in Section~\ref{section:low_rank_hadamard}, \ref{section:low_rank_kronecker}, \ref{section:lrank_khatri_decom}; while Section~\ref{sec:lora_in_llm} describes the application of low-rank matrix decomposition in fine-tuning large language models.

\section{Low-Rank Decomposition via Alternating Least Squares}\label{sec:als}

We then formally examine algorithms designed to address the following problem: The matrix $\bA$ is approximately factorized into an $M\times K$ matrix $\bW$ and a $K \times  N$ matrix $\bZ$. 
Typically, $K$ is selected to be smaller than both $M$ and $N$, thereby  ensuring that $\bW$ and $\bZ$ have reduced dimensions compared to the original  matrix $\bA$. 
This reduction in dimensionality yields a compressed representation  of the original data matrix. 
An appropriate decision on the value of $K$ is crucial in practice, although this choice often depends on the specific problem at hand.
The significance of the factorization lies in the fact that if we denote the column partitions of $\bA$ and $\bZ$ as $\bA=[\ba_1, \ba_2, \ldots, \ba_N]$ and $\bZ=[\bz_1, \bz_2, \ldots, \bz_N]$, respectively, then each column $\ba_n$ can be approximated as $\ba_n = \bW\bz_n$. In other words, each column $\ba_n$ is approximated as a linear combination of the columns of $\bW$, with the coefficients given by the corresponding elements in $\bz_n$. 
Consequently, the columns of $\bW$ can be considered as forming the column basis of $\bA$.

To achieve the approximation $\bA\approx\bW\bZ$, we need to define a loss function that quantifies the discrepancy between $\bA$ and $\bW\bZ$. 
In our context, the chosen loss function is the Frobenius norm of the difference between the two matrices, which vanishes to zero if $\bA=\bW\bZ$. The benefits of this choice will become clear soon.

To simplify the problem, let's first assume that there are no missing ratings. 
Our goal is to project data vectors $\ba_n\in\real^M$ into a lower dimension $\bz_n \in \real^K$, where $K<\min\{M, N\}$,
in such a way that the \textit{reconstruction error} (a.k.a., objective function, cost function, or loss function) as measured by the Frobenius norm is minimized (assume $K$ is known):
\begin{equation}\label{equation:als-per-example-loss}
	\mathop{\min}_{\bW,\bZ}  \sum_{n=1}^N \sum_{m=1}^{M} \left(a_{mn} - \bw_m^\top\bz_n\right)^2,
\end{equation}
where $\bW=[\bw_1^\top; \bw_2^\top; \ldots; \bw_M^\top]\in \real^{M\times K}$ and $\bZ=[\bz_1, \bz_2, \ldots, \bz_N] \in \real^{K\times N}$ containing $\bw_m$'s and $\bz_n$'s as \textbf{rows and columns}, respectively. The loss formulation in Equation~\eqref{equation:als-per-example-loss} is referred to as the \textit{per-example loss}. 
And it can be equivalently expressed as
$$
L(\bW,\bZ) = \sum_{n=1}^N \sum_{m=1}^{M} \left(a_{mn} - \bw_m^\top\bz_n\right)^2 
= \norm{\bW\bZ-\bA}_F^2.
$$ 
Moreover, the loss function $L(\bW,\bZ)=\sum_{n=1}^N \sum_{m=1}^{M} \left(a_{mn} - \bw_m^\top\bz_n\right)$ exhibits convexity with respect to $\bZ$ when $\bW$ is held constant, and similarly, with respect to $\bW$ when $\bZ$ is fixed. 
This property   motivates the alternate algorithm that alternately fixes one of the variables and optimizes over the other.
Consequently, we can first minimize it with respect to $\bZ$ while keeping $\bW$ fixed, and subsequently minimize it with respect to $\bW$ with $\bZ$ fixed.
This approach leads to two optimization problems, which we will refer to as ALS1 and ALS2, respectively:
$$
\left\{
\begin{aligned}
	\bZ &\leftarrow \mathop{\arg \min}_{\bZ} L(\bW,\bZ);    \qquad \text{(ALS1)} \\ 
	\bW &\leftarrow \mathop{\arg \min}_{\bW} L(\bW,\bZ). \qquad \text{(ALS2)}
\end{aligned}
\right.
$$
This method is known as the \textit{coordinate descent algorithm}, where we alternate between optimizing the least squares with respect to $\bW$ and $\bZ$. 
Hence, it is also called the \textit{alternating least squares (ALS)} algorithm \citep{comon2009tensor, takacs2012alternating, giampouras2018alternating}. Convergence is guaranteed if the loss function $L(\bW,\bZ)$ decreases at each iteration, and we will delve further into this topic later.

\begin{remark}[Convexity and Global Minimum]
Although the loss function defined by Frobenius norm $\norm{\bW\bZ-\bA}^2$ is convex either with respect to  $\bW$ when $\bZ$ is fixed or vice versa (called \textit{marginally convex}), it is not \textit{jointly convex} in both variables simultaneously. 
As a result, finding the global minimum is infeasible. 
However, the convergence to a local minimum is assured.
\end{remark}

\subsection*{Given $\bW$, Optimizing $\bZ$}
Now, let's consider the problem of $\bZ \leftarrow \mathop{\arg \min}_{\bZ} L(\bW,\bZ)$. When there exists a unique minimum of the loss function $L(\bW,\bZ)$ with respect to $\bZ$, we refer to it as the \textit{least squares} minimizer of $\mathop{\arg \min}_{\bZ} L(\bW,\bZ)$. 
With $\bW$ fixed, $L(\bW,\bZ)$  can be denoted  as $L(\bZ|\bW)$ (or more compactly, as $L(\bZ)$) to emphasize  its  dependence on $\bZ$:
$$
\begin{aligned}
	L(\bZ|\bW) &= \norm{\bW\bZ-\bA}_F^2= \left\Vert\bW[\bz_1,\bz_2,\ldots, \bz_N]-[\ba_1,\ba_2,\ldots,\ba_N]\right\Vert^2=\left\Vert
	\begin{bmatrix}
		\bW\bz_1 - \ba_1 \\
		\bW\bz_2 - \ba_2\\
		\vdots \\
		\bW\bz_N - \ba_N
	\end{bmatrix}
	\right\Vert^2. 
\end{aligned}
$$
Now, if we define 
$$
\widetildebW = 
\begin{bmatrix}
	\bW & \bzero & \ldots & \bzero\\
	\bzero & \bW & \ldots & \bzero\\
	\vdots & \vdots & \ddots & \vdots \\
	\bzero & \bzero & \ldots & \bW
\end{bmatrix}
\in \real^{MN\times KN}, 
\gap 
\widetildebz=
\begin{bmatrix}
	\bz_1 \\ \bz_2 \\ \vdots \\ \bz_N
\end{bmatrix}
\in \real^{KN},
\gap 
\widetildeba=
\begin{bmatrix}
	\ba_1 \\ \ba_2 \\ \vdots \\ \ba_N
\end{bmatrix}
\in \real^{MN},
$$
then the (ALS1) problem can be reduced to the ordinary least squares (OLS)  problem of minimizing  $\norm{\widetildebW \widetildebz - \widetildeba}^2$ with respect to $\widetildebz$. The solution is given by 
$$
\widetildebz = (\widetildebW^\top\widetildebW)^{-1} \widetildebW^\top\widetildeba.
$$
However, it is not advisable  to obtain the result using this approach, as computing the inverse of  $\widetildebW^\top\widetildebW$ requires $2(KN)^3$ flops (floating-point operations; see, for example, Chapter 2 of \citet{lu2021numerical}).
Alternatively, a more efficient way to solve the problem of (ALS1) is to find the gradient of $L(\bZ|\bW)$ with respect to $\bZ$ (assuming all the partial derivatives of this function exist):
\begin{equation}\label{equation:givenw-update-z-allgd}
\begin{aligned}
\nabla L(\bZ|\bW) &= 
\frac{\partial \,\,\trace\left((\bW\bZ-\bA)(\bW\bZ-\bA)^\top\right)}{\partial \bZ}\\
&=\frac{\partial \,\,\trace\left((\bW\bZ-\bA)(\bW\bZ-\bA)^\top\right)}{\partial (\bW\bZ-\bA)}
\frac{\partial (\bW\bZ-\bA)}{\partial \bZ}\\
&\stackrel{\star}{=}2  \bW^\top(\bW\bZ-\bA) \in \real^{K\times N},
\end{aligned}
\end{equation}
where the first equality arises from the definition of the Frobenius norm  such that $\norm{\bA} = \sqrt{\sum_{m=1,n=1}^{M,N} (a_{mn})^2}=\sqrt{\trace(\bA\bA^\top)}$, and equality ($\star$) follows from the fact that $\frac{\partial \trace(\bA\bA^\top)}{\partial \bA} = 2\bA$. When the loss function is a differentiable function of $\bZ$, we can determine the least squares solution using differential calculus. 
Since we optimize over an open set $\real^{K\times N}$,  a minimum of the function 
$L(\bZ|\bW)$ must be a root of the equation:
$$
\nabla L(\bZ|\bW)  = \bzero.
$$
By finding the root of the  equation above, we obtain the ``candidate" update for $\bZ$ that identifies  the minimizer of $L(\bZ|\bW)$:
\begin{equation}\label{equation:als-z-update}
	{\bZ = (\bW^\top\bW)^{-1} \bW^\top \bA  \leftarrow \mathop{\arg \min}_{\bZ} L(\bZ|\bW).}
\end{equation}
This computation takes $2K^3$ flops to compute the inverse of $\bW^\top\bW$, in contrast to $2(KN)^3$ flops needed to compute the inverse  of $\widetildebW^\top\widetildebW$.
Before confirming that a root of the equation above is indeed a minimizer (as opposed to a maximizer, hence the term ``candidate" update), it is essential  to establish the convexity of the function. 
In  case where  the function is twice differentiable, this verification  can be equivalently achieved by confirming (see, for example, \citet{beck2014introduction, lu2021numerical}): 
$$
\nabla^2 L(\bZ|\bW) > 0.~
\footnote{In short, a twice continuously differentiable function $f$ over an open convex set $\sS$ is called \textit{convex} if and only if $\nabla^2f(\bx)\geq \bzero $ for any $\bx\in \sS$ (sufficient and necessary for convex); and called \textit{strictly convex} if $\nabla^2f(\bx)> \bzero$ for any $\bx\in \sS$ (only sufficient for strictly convex, e.g., $f(x)=x^6$ is strictly convex, but $f^{\prime\prime}(x)=30x^4$ is equal to zero at $x=0$.). 
	And when the convex function $f$ is a continuously differentiable function over a convex set $\sS$, the stationary point $\nabla f(\bx^\star)=\bzero$ of $\bx^\star\in\sS$ is  a \textit{global minimizer} of $f$ over $\sS$.
	In our context, when given $\bW$ and updating $\bZ$, the function is defined over the entire space $\real^{K\times N}$ (an open set).
}
$$
That is, the Hessian matrix is positive definite. To demonstrate this, we explicitly express the Hessian matrix as
\begin{equation}\label{equation:als-z-update_hessian}
	\nabla^2 L(\bZ|\bW)= 2\widetildebW^\top\widetildebW \in \real^{KN\times KN},
\end{equation}
which maintains full rank if $\bW\in \real^{M\times K}$ has full rank  and $K<M$.

\begin{remark}[Positive Definite Hessian if $\bW$ Has Full Rank]\label{remark:als_rmk1}
	We  claim that if $\bW\in\real^{M\times K}$ has full rank $K$ with $K<M$, then $\nabla^2 L(\bZ|\bW)$ is positive definite. This can be demonstrated by confirming that when $\bW$ has full rank, the equation $\bW\bx=\bzero$  holds true only when $\bx=\bzero$, since the null space of $\bW$ has a  dimension of 0. Therefore, 
	$$
	\bx^\top (2\bW^\top\bW)\bx >0, \qquad \text{for any nonzero vector $\bx\in \real^K$}.
	$$ 
\end{remark}
Now, the issue is that we need to confirm \textbf{whether $\bW$ has full rank so that the Hessian of $L(\bZ|\bW)$ is positive definite}; otherwise, we cannot claim that the update of $\bZ$ in Equation~\eqref{equation:als-z-update} reduces the loss (due to convexity) so that the matrix decomposition progressively improves the approximation of the original matrix $\bA$ by $\bW\bZ$ in each iteration.
We will address the positive definiteness of the Hessian matrix shortly, relying on the following lemma.

\begin{lemma}[Rank of $\bZ$ after Updating]\label{lemma:als-update-z-rank}
Suppose $\bA\in \real^{M\times N}$ has full rank with \textcolor{mylightbluetext}{$M\leq N$} and $\bW\in \real^{M\times K}$ has full rank with $K<M$ (i.e., $K<M\leq N$), then the update of $\bZ=(\bW^\top\bW)^{-1} \bW^\top \bA \in \real^{K\times N}$ in Equation~\eqref{equation:als-z-update} has full rank.
\end{lemma}
\begin{proof}[of Lemma~\ref{lemma:als-update-z-rank}]
Since $\bW^\top\bW\in \real^{K\times K}$ has full rank if $\bW$ has full rank, it follows that $(\bW^\top\bW)^{-1} $ has full rank. 

Suppose $\bW^\top\bx=\bzero$, this implies that $(\bW^\top\bW)^{-1} \bW^\top\bx=\bzero$. Thus, the following two null spaces satisfy~\footnote{We denote the null space of matrix $\bA$ by $\nspace(\bA)$ and its column space by $\cspace(\bA)$.}:
$
\nspace(\bW^\top) \subseteq \nspace\left((\bW^\top\bW)^{-1} \bW^\top\right).
$
Moreover, suppose $\bx$ lies in the null space of $(\bW^\top\bW)^{-1} \bW^\top$ such that $(\bW^\top\bW)^{-1} \bW^\top\bx=\bzero$. And since $(\bW^\top\bW)^{-1} $ is invertible, it implies $ \bW^\top\bx=(\bW^\top\bW)\bzero=\bzero$, leading to
$
\nspace\left((\bW^\top\bW)^{-1} \bW^\top\right)\subseteq \nspace(\bW^\top).
$
Consequently,  it follows that 
\begin{equation}\label{equation:als-z-sandiwch1}
\nspace(\bW^\top) = \nspace\left((\bW^\top\bW)^{-1} \bW^\top\right).
\end{equation}

Therefore, $(\bW^\top\bW)^{-1} \bW^\top$ has full rank $K$. Let $\bT=(\bW^\top\bW)^{-1} \bW^\top\in \real^{K\times M}$, and suppose $\bT^\top\bx=\bzero$. This implies $\bA^\top\bT^\top\bx=\bzero$, yielding 
$
\nspace(\bT^\top) \subseteq \nspace(\bA^\top\bT^\top).
$
Similarly, suppose $\bA^\top(\bT^\top\bx)=\bzero$. Since $\bA$ has full rank with the dimension of the null space being 0: $\dim\left(\nspace(\bA^\top)\right)=0$, $(\bT^\top\bx)$ must be zero. The claim follows  since $\bA$ has full rank $M$ with the row space of $\bA^\top$ being equal to the column space of $\bA$, where $\dim\left(\cspace(\bA)\right)=M$ and the $\dim\left(\nspace(\bA^\top)\right) = M-\dim\left(\cspace(\bA)\right)=0$. 
Consequently, $\bx$ is in the null space of $\bT^\top$ if $\bx$ is in the null space of $\bA^\top\bT^\top$:
$
\nspace(\bA^\top\bT^\top)\subseteq \nspace(\bT^\top).
$
Hence we obtain
\begin{equation}\label{equation:als-z-sandiwch2}
	\nspace(\bT^\top) = \nspace(\bA^\top\bT^\top).
\end{equation}
Since $\bT^\top$ has full rank $K<M\leq N$, it follows that $\dim\left(\nspace(\bT^\top) \right) = \dim\left(\nspace(\bA^\top\bT^\top)\right)=0$.
Therefore,
$\bZ^\top=\bA^\top\bT^\top$ has full rank $K$.
We complete the proof.
\end{proof}

\subsection*{Given $\bZ$, Optimizing $\bW$}

With $\bZ$ fixed, $L(\bW,\bZ)$ can be expressed as $L(\bW|\bZ)$ (or more compactly, as $L(\bW)$)  to emphasize its dependence on $\bW$:
$$
\begin{aligned}
	L(\bW|\bZ) &= \norm{\bW\bZ-\bA}_F^2.
\end{aligned}
$$
To solve the optimization problem (ALS2) directly, we need to compute the gradient of  $L(\bW|\bZ)$ with respect to $\bW$:
$$
\begin{aligned}
\nabla L(\bW|\bZ) &= 
2(\bW\bZ-\bA)\bZ^\top \in \real^{M\times K}.
\end{aligned}
$$
Similarly, the ``candidate" update for  $\bW$ can be obtained by locating the root of the gradient $\nabla L(\bW|\bZ)$:
\begin{equation}\label{equation:als-w-update}
	{\bW^\top = (\bZ\bZ^\top)^{-1}\bZ\bA^\top  \leftarrow \mathop{\arg\min}_{\bW} L(\bW|\bZ).}
\end{equation}
Once more, we emphasize that the update is merely a ``candidate" update. 
Further validation is necessary  to ascertain the positive definiteness of the Hessian matrix.
The Hessian matrix is obtained as follows:
\begin{equation}\label{equation:als-w-update_hessian}
	\begin{aligned}
		\nabla^2 L(\bW|\bZ) =2\widetildebZ\widetildebZ^\top \in \real^{KM\times KM}.
	\end{aligned}
\end{equation}
Therefore, by analogous analysis, if $\bZ$ has full rank with $K<N$, the Hessian matrix is positive definite.

In Lemma~\ref{lemma:als-update-z-rank}, we proved that $\bZ$ has full rank under certain conditions, such that  the Hessian matrix in Equation~\eqref{equation:als-w-update_hessian} is positive definite, and the update in Equation~\eqref{equation:als-w-update} exists.
We here prove that $\bW$ also has full rank under certain conditions, such that the Hessian in Equation~\eqref{equation:als-z-update_hessian} is positive definite, and the update in  Equation~\eqref{equation:als-z-update} exists.
\begin{lemma}[Rank of $\bW$ after Updating]\label{lemma:als-update-w-rank}
	Suppose $\bA\in \real^{M\times N}$ has full rank with \textcolor{mylightbluetext}{$M\geq N$} and $\bZ\in \real^{K\times N}$ has full rank with $K<N$ (i.e., $K<N\leq M$), then the update of $\bW^\top = (\bZ\bZ^\top)^{-1}\bZ\bA^\top$ in Equation~\eqref{equation:als-w-update} has full rank.
\end{lemma}
The proof of Lemma~\ref{lemma:als-update-w-rank} is similar to that of  Lemma~\ref{lemma:als-update-z-rank}, and we shall not repeat the details.

\paragraph{Key observation.}
Combining the observations in Lemma~\ref{lemma:als-update-z-rank} and Lemma~\ref{lemma:als-update-w-rank}, as long as we \textcolor{mylightbluetext}{initialize $\bZ$ and $\bW$ to have full rank}, the updates in Equation~\eqref{equation:als-z-update} and Equation~\eqref{equation:als-w-update} are reasonable \textbf{since the Hessians in Equation~\eqref{equation:als-z-update_hessian} and \eqref{equation:als-w-update_hessian} are positive definite}. 
Note that we need an additional condition to satisfy  
both Lemma~\ref{lemma:als-update-z-rank} 
and Lemma~\ref{lemma:als-update-w-rank}: $M=N$, meaning there must be an equal number of movies and   users (in the Netflix context).
We will relax this condition in the next section through regularization.
We summarize the process in Algorithm~\ref{alg:als}.

\begin{algorithm}[h] 
	\caption{Alternating Least Squares (ALS)}
	\label{alg:als}
	\begin{algorithmic}[1] 
		\Require Matrix $\bA\in \real^{M\times N}$ \textcolor{mylightbluetext}{with $M= N$};
		\State Initialize $\bW\in \real^{M\times K}$, $\bZ\in \real^{K\times N}$ \textcolor{mylightbluetext}{with full rank and $K<M= N$}; 
		\State Choose a stop criterion on the approximation error $\delta$;
		\State Choose the maximal number of iterations $C$;
		\State $iter=0$; \Comment{Count for the number of iterations}
		\While{$\norm{\bA-\bW\bZ}_F>\delta $ and $iter<C$} 
		\State $iter=iter+1$;
		\State $\bZ = (\bW^\top\bW)^{-1} \bW^\top \bA  \leftarrow \mathop{\arg \min}_{\bZ} L(\bZ|\bW)$;
		\State $\bW^\top = (\bZ\bZ^\top)^{-1}\bZ\bA^\top  \leftarrow \mathop{\arg\min}_{\bW} L(\bW|\bZ)$;
		\EndWhile
		\State Output $\bW$ and $\bZ$;
	\end{algorithmic} 
\end{algorithm}

\subsection{Regularization: Extension to General Matrices}\label{section:regularization-extention-general}

\textit{Regularization} is a machine learning technique used to prevent overfitting and enhance model generalization. Overfitting happens when a model becomes too complex and fits the training data excessively, leading to poor performance on new, unseen data.
To mitigate this, regularization adds a constraint or penalty term to the loss function used for model optimization, discouraging the development of overly complex models. This creates a balance between model simplicity and the ability to fit the training data well.
Common forms of regularization include $\ell_1$ regularization, $\ell_2$ regularization, and elastic net regularization (which combines $\ell_1$ and $\ell_2$ regularization). 
Regularization is widely applied in machine learning algorithms like linear regression, logistic regression, and neural networks.

In the context of the alternating least squares problem, we can introduce an $\ell_2$ regularization term  to minimize the following loss:
\begin{equation}\label{equation:als-regularion-full-matrix}
	L(\bW,\bZ)  =\norm{\bW\bZ-\bA}_F^2 +\lambda_w \norm{\bW}_F^2 + \lambda_z \norm{\bZ}_F^2, \qquad \lambda_w>0, \lambda_z>0,
\end{equation}
where the gradient with respect to $\bZ$ and $\bW$ are given respectively by 
\begin{equation}\label{equation:als-regulari-gradien}
	\left\{
	\begin{aligned}
		\nabla L(\bZ|\bW) &= 2\bW^\top(\bW\bZ-\bA) + 2\lambda_z\bZ \in \real^{K\times N};\\
		\nabla L(\bW|\bZ)  &= 2(\bW\bZ-\bA)\bZ^\top + 2\lambda_w\bW \in \real^{M\times K}.
	\end{aligned}
	\right.
\end{equation}
The Hessian matrices are given respectively by 
$$
\left\{
\begin{aligned}
	\nabla^2 L(\bZ|\bW) &= 2\widetildebW^\top\widetildebW+ 2\lambda_z\bI \in \real^{KN\times KN};\\
	\nabla^2 L(\bW|\bZ)  &= 2\widetildebZ\widetildebZ^\top + 2\lambda_w\bI \in \real^{KM\times KM}, \\
\end{aligned}
\right.
$$
which are positive definite due to the perturbation by the regularization. 
\textbf{The regularization ensues that the Hessian matrices become positive definite, even if $\bW$ and $\bZ$ are rank-deficient}. 
Consequently, matrix decomposition can be extended to any matrix, regardless of whether $M>N$ or $M<N$. In rare cases, $K$ can be chosen as $K>\max\{M, N\}$ such that a \textit{high-rank approximation} of $\bA$ is obtained. However, in most scenarios, we aim to find the \textit{low-rank approximation} of $\bA$ with $K<\min\{M, N\}$. For example,  ALS can be employed to discover low-rank neural networks, reducing the memory usage of  neural networks while improving performance \citep{lu2021numerical}.
Therefore, the minimizers can be determined by identifying the roots of the gradient:
\begin{equation}\label{equation:als-regular-final-all}
	\left\{
	\begin{aligned}
		\bZ &= (\bW^\top\bW+ \lambda_z\bI)^{-1} \bW^\top \bA  ;\\
		\bW^\top &= (\bZ\bZ^\top+\lambda_w\bI)^{-1}\bZ\bA^\top .
	\end{aligned}
	\right.
\end{equation}
The regularization parameters $\lambda_z, \lambda_w\in \real$ are used to balance the trade-off
between the accuracy of the approximation and the smoothness of the computed solution. 
These parameters are typically chosen based on the specific problem at hand and can be determined using \textit{cross-validation} (CV). 
We outline the process in Algorithm~\ref{alg:als-regularizer}.

The $\ell_2$ (or $\ell_1$ regularizations) can be applied to generalize the ALS problem to general matrices.
However, we will address the issue of missing entries in the following sections, thus the problem becomes a matrix completion formulation.
In this context, the $\ell_1$ and $\ell_2$ regularizations  are not the only options; for instance, the \textit{nuclear norm} \footnote{Also called the \textit{Schatten 1-norm} or \textit{trace norm}.} of $\bW\bZ$ (the sum of singular values of the matrix) can be used, for which the \textit{Soft-Impute for matrix completion} algorithm guarantees the recovery of the matrix when the number of observed entries $z$ satisfies
$$
z\geq C r n \log n,
$$
where the underlying matrix $\bA$ is of size $\real^{n\times n}$ and $C > 0$ is a fixed universal constant \citep{gross2011recovering, hastie2015statistical}.

\begin{algorithm}[H] 
\caption{Alternating Least Squares with Regularization}
\label{alg:als-regularizer}
\begin{algorithmic}[1] 
\Require Matrix $\bA\in \real^{M\times N}$;
\State Initialize $\bW\in \real^{M\times K}$, $\bZ\in \real^{K\times N}$ \textcolor{mylightbluetext}{randomly without condition on the rank and the relationship between $M, N, K$}; 
\State Choose a stop criterion on the approximation error $\delta$;
\State Choose regularization parameters $\lambda_w, \lambda_z$;
\State Choose the  maximal number of iterations $C$;
\State $iter=0$; \Comment{Count for the number of iterations}
\While{$\norm{\bA-\bW\bZ}_F>\delta $ and $iter<C$}
\State $iter=iter+1$; 
\State $\bZ = (\bW^\top\bW+ \lambda_z\bI)^{-1} \bW^\top \bA  \leftarrow \mathop{\arg \min}_{\bZ} L(\bZ|\bW)$;
\State $\bW^\top = (\bZ\bZ^\top+\lambda_w\bI)^{-1}\bZ\bA^\top  \leftarrow \mathop{\arg\min}_{\bW} L(\bW|\bZ)$;
\EndWhile
\State Output $\bW$ and $\bZ$;
\end{algorithmic} 
\end{algorithm}

\subsection{Missing Entries and Rank-One Update}\label{section:alt-columb-by-column}
The matrix decomposition via the ALS method is extensively used in the context of Netflix recommender data, where a substantial number of entries are missing due to users not having watched certain movies or choosing not to rate them for various reasons.
In this scenario, the low-rank matrix decomposition is also known as \textit{matrix completion}, which can help recover unobserved entries \citep{jain2017non}.
To address this, we can introduce an additional mask matrix $\bM\in \{0,1\}^{M\times N}$, where $m_{mn}\in \{0,1\}$ indicates if the user $n$ has rated the movie $m$ or not. Therefore, the loss function can be defined as 
$$
L(\bW,\bZ) = \norm{\bM\hadaprod  \bA- \bM\hadaprod (\bW\bZ)}_F^2,
$$
where $\hadaprod$ represents the \textit{Hadamard product} between matrices. For example, the Hadamard product of a $3 \times 3$ matrix $\bA$ and a $3\times 3$ matrix $\bB$ is
$$
\bA\hadaprod \bB = 
\begin{bmatrix}
	a_{11} & a_{12} & a_{13} \\
	a_{21} & a_{22} & a_{23} \\
	a_{31} & a_{32} & a_{33}
\end{bmatrix}
\hadaprod
\begin{bmatrix}
	b_{11} & b_{12} & b_{13} \\
	b_{21} & b_{22} & b_{23} \\
	b_{31} & b_{32} & b_{33}
\end{bmatrix}
=
\begin{bmatrix}
	a_{11}b_{11} &  a_{12}b_{12} & a_{13}b_{13} \\
	a_{21}b_{21} & a_{22}b_{22} & a_{23}b_{23} \\
	a_{31}b_{31} & a_{32}b_{32} & a_{33}b_{33}
\end{bmatrix}.
$$
To find the solution of the problem, we decompose the updates in Equation~\eqref{equation:als-regular-final-all} into:
\begin{equation}\label{equation:als-ori-all-wz}
	\left\{
	\begin{aligned}
		\bz_n &= (\bW^\top\bW+ \lambda_z\bI)^{-1} \bW^\top \ba_n, &\gap& \text{for $n\in \{1,2,\ldots, N\}$}  ;\\
		\bw_m &= (\bZ\bZ^\top+\lambda_w\bI)^{-1}\bZ\bb_m,  &\gap& \text{for $m\in \{1,2,\ldots, M\}$} ,
	\end{aligned}
	\right.
\end{equation}
where $\bZ=[\bz_1, \bz_2, \ldots, \bz_N]$ and $\bA=[\ba_1,\ba_2, \ldots, \ba_N]$ represent the column partitions of $\bZ$ and $\bA$, respectively. Similarly, $\bW^\top=[\bw_1, \bw_2, \ldots, \bw_M]$ and $\bA^\top=[\bb_1,\bb_2, \ldots, \bb_M]$ are the column partitions of $\bW^\top$ and $\bA^\top$, respectively, for further evaluation. This decomposition of the updates indicates that they can be performed in a column-by-column fashion (the rank-one update).

\paragraph{Given $\bW$.}
Let $\bo_n\in \{0,1\}^M$ represent the movies rated by user $n$, where $o_{nm}=1$ if user $n$ has rated movie $m$, and $o_{nm}=0$ otherwise. Then the $n$-th column of $\bA$ without missing entries can be denoted using the Matlab-style notation as $\ba_n[\bo_n]$. 
And we want to approximate the existing entries of the $n$-th column by $\ba_n[\bo_n] \approx \bW[\bo_n, :]\bz_n$, which is essentially  a rank-one least squares problem:
\begin{equation}\label{equation:als-ori-all-wz-modif-z}
	\begin{aligned}
		\bz_n &= \left(\bW[\bo_n, :]^\top\bW[\bo_n, :]+ \lambda_z\bI\right)^{-1} \bW[\bo_n, :]^\top \ba_n[\bo_n], &\gap& \text{for $n\in \{1,2,\ldots, N\}$} .
	\end{aligned}
\end{equation}

\paragraph{Given $\bZ$.}
Similarly, if $\bp_m \in\{0,1\}^{N}$ denotes the users who have rated  movie $m$, with $p_{mn}=1$ if the movie $m$ has been rated by user $n$, and $p_{mn}=0$ otherwise. Then the $m$-th row of $\bA$ without missing entries can be denoted by the Matlab-style notation as $\bb_m[\bp_m]$. And we want to approximate the existing entries of the $m$-th row by $\bb_m[\bp_m] \approx \bZ[:, \bp_m]^\top\bw_m$, 
\footnote{Note that $\bZ[:, \bp_m]^\top$ is the transpose of $\bZ[:, \bp_m]$, which is equal to $\bZ^\top[\bp_m,:]$, i.e., transposing first and then selecting.}
which  is a rank-one least squares problem again:
\begin{equation}\label{equation:als-ori-all-wz-modif-w}
	\begin{aligned}
		\bw_m &= (\bZ[:, \bp_m]\bZ[:, \bp_m]^\top+\lambda_w\bI)^{-1}\bZ[:, \bp_m]\bb_m[\bp_m],  &\gap& \text{for $m\in \{1,2,\ldots, M\}$} .
	\end{aligned}
\end{equation}
The procedure is once again presented in Algorithm~\ref{alg:als-regularizer-missing-entries}.
\begin{algorithm}[h] 
\caption{Alternating Least Squares with Missing Entries and Regularization}
\label{alg:als-regularizer-missing-entries}
\begin{algorithmic}[1] 
\Require Matrix $\bA\in \real^{M\times N}$;
\State Initialize $\bW\in \real^{M\times K}$, $\bZ\in \real^{K\times N}$ \textcolor{mylightbluetext}{randomly without condition on the rank and the relationship between $M, N, K$}; 
\State Choose a stop criterion on the approximation error $\delta$;
\State Choose regularization parameters $\lambda_w, \lambda_z$;
\State Compute the mask matrix $\bM$ from $\bA$;
\State Choose the maximal number of iterations $C$;
\State $iter=0$; \Comment{Count for the number of iterations}
\While{\textcolor{mylightbluetext}{$\norm{\bM\hadaprod  \bA- \bM\hadaprod (\bW\bZ)}_F^2>\delta $} and $iter<C$}
\State $iter=iter+1$; 
\For{$n=1,2,\ldots, N$}
\State $\bz_n = \left(\bW[\bo_n, :]^\top\bW[\bo_n, :]+ \lambda_z\bI\right)^{-1} \bW[\bo_n, :]^\top \ba_n[\bo_n]$; \Comment{$n$-th column of $\bZ$}
\EndFor

\For{$m=1,2,\ldots, M$}
\State $\bw_m = (\bZ[:, \bp_m]\bZ[:, \bp_m]^\top+\lambda_w\bI)^{-1}\bZ[:, \bp_m]\bb_m[\bp_m]$;\Comment{$m$-th column of $\bW^\top$}
\EndFor
\EndWhile
\State Output $\bW^\top=[\bw_1, \bw_2, \ldots, \bw_M]$ and $\bZ=[\bz_1, \bz_2, \ldots, \bz_N]$;
\end{algorithmic} 
\end{algorithm}

\section{Special Matrix Products and Properties}\label{sec:special}
In the forthcoming sections, several key matrix products will play a significant role in explaining the algorithms under discussion. These matrix operations are not only foundational elements in the development and understanding of these algorithms but also find broad application in the context of tensor decomposition \citep{lu2021numerical}. 
The significance of these matrix products cannot be overstated, as they form the core of many computational techniques and methodologies that are central to both algorithmic illustration and tensor-based analyses. These matrix products provide the mathematical framework for decomposing complex data structures into simpler, more interpretable components, and they enable the efficient implementation of algorithms designed to extract meaningful information from high-dimensional data.

\subsection{Kronecker Product}
The \textit{Kronecker product} of two vectors $\ba \in \real^{I}$ and $\bb\in \real^{K}$, denoted by $\ba\otimes \bb$, is defined as follows:
$$
\ba\otimes \bb=
\begin{bmatrix}
	a_1\bb \\
	a_2\bb \\
	\vdots \\
	a_I\bb
\end{bmatrix}
\in\real^{IK},
$$
which is a column vector of size $(IK)$. It can be easily verified that if $\norm{\ba}=\norm{\bb}=1$, then $\norm{\ba\otimes \bb}=1$.
\begin{definition}[Matrix Kronecker Product]\label{definition:kronecker-product}
	Similarly, the \textit{Kronecker product} of two matrices $\bA \in \real^{I\times J}$ and $\bB\in \real^{K\times L}$, denoted by $\bA\otimes \bB$, is defined as follows:
	$$
	\begin{aligned}
		\bA\otimes \bB &= 
		\begin{bmatrix}
			a_{11} \bB & a_{12}\bB & \ldots & a_{1J}\bB \\
			a_{21} \bB & a_{22}\bB & \ldots & a_{2J}\bB \\
			\vdots  & \vdots  & \ddots & \vdots \\
			a_{I1} \bB & a_{I2}\bB & \ldots & a_{IJ}\bB \\
		\end{bmatrix}\\
		&=
		\begin{bmatrix}
			\ba_1 \otimes \bb_1 \,\,\ldots & \ba_1\otimes \ba_L \mid 
			\ba_2 \otimes \bb_1   \,\,\ldots &\ba_2\otimes \ba_L \mid 
			\ba_J \otimes \bb_1   \,\,\ldots &\ba_J\otimes \ba_L\,
		\end{bmatrix},
	\end{aligned}
	$$
	which is a matrix of size $(IK)\times (JL)$. 
	That is, the Kronecker product $\bA\otimes \bB$ can be divided into $I\times J$ blocks; for each block $(i,j)$, it is a $K\times L$ matrix recorded by $a_{ij}\bB$.
	When $\bA$ and $\bB$ are of the same shape, then it can be shown that $\bA\hadaprod \bB$ is a principal submatrix of $\bA\otimes \bB$.
\end{definition}

The Kronecker product maintains specific matrix properties that can simplify computations and ensure stability.
\begin{lemma}[Kronecker of Orthogonal, Triangular, Diagonal, (Semi)definite, Nonsingular]\label{lemma:krokecker_keep_special}
	The Kronecker products of two orthogonal, two triangular, two diagonal,  two (semi)definite, or two nonsingular matrices are also orthogonal, triangular,  diagonal, (semi)definite, or nonsingular, respectively. 
\end{lemma}

Specifically, we notice that, given four vectors \{$\ba \in \real^{I}$ and $\bb\in \real^{K}$\} and \{$\bc \in \real^{I}$ and $\bd\in \real^{K}\}$, then 
\begin{equation}\label{equation:kronecker-vector-find2}
	(\ba\otimes \bb)^\top(\bc\otimes \bd)=
	\begin{bmatrix}
		a_1\bb^\top & 
		a_2\bb^\top &
		\ldots &
		a_I\bb^\top  
	\end{bmatrix}
	\begin{bmatrix}
		c_1\bd \\
		c_2\bd \\
		\vdots \\
		c_I\bd
	\end{bmatrix}
	=\sum_{i=1}^{I}a_ic_i \bb^\top\bd=(\ba^\top\bc)(\bb^\top\bd).
\end{equation}
Particularly, when $\bc=\ba$ and $\bd=\bb$, it follows that 
$$
(\ba\otimes \bb)^\top(\ba\otimes \bb) =\norm{\ba}^2 \cdot \norm{\bb}^2.
$$
Similarly, given four matrices $\bA,\bC\in \real^{I\times J}$ and $\bB,\bD\in \real^{K\times L}$, it follows that 
\begin{equation}\label{equation:kron_pro1_trans}
	(\bA\otimes \bB)^\top (\bC\otimes \bD) = (\bA^\top\bC) \otimes (\bB^\top\bD).
\end{equation}
Note also, for $\bA \in \real^{I\times J}$, $\bB\in \real^{K\times L}$, $\bC\in \real^{J\times I}$, and $\bD\in \real^{L\times K}$, the  equation above reduces to
\begin{equation}\label{equation:kron_pro1_notrans}
	(\bA\otimes \bB) (\bC\otimes \bD) = (\bA\bC) \otimes (\bB\bD).
\end{equation}
When $\bA \in \real^{I\times J}$, $\bB\in \real^{K\times L}$, $\bc\in \real^{J}$, and $\bd\in \real^{L}$, the equality becomes
\begin{equation}\label{equation:kronecker_eq3}
	(\bA\otimes \bB) (\bc \otimes \bd) = (\bA\bc) \otimes (\bB\bd).
\end{equation}
More generally, when $\bA \in \real^{I\times J}$, $\bB\in \real^{K\times L}$, $\bC\in \real^{J\times P}$, and $\bD\in \real^{L\times Q}$, we have
\begin{equation}
	(\bA\otimes \bB) (\bC\otimes \bD) = (\bA\bC) \otimes (\bB\bD).
\end{equation}

From the definition of the Kronecker product, it can be readily verified that any eigenvalue of $\bA\otimes \bB$ arises as a product of the eigenvalues of $\bA$ and $\bB$.
\begin{lemma}[Eigenvalue of Kronecker Product, \citep{horn1994topics}]
	Suppose $\bA\in\real^{m\times m}$ has an eigenpair $(\lambda, \bx)$, and $\bB\in\real^{n\times n}$ has an eigenpair $(\mu, \by)$. 
	Then $\lambda\mu$ is an eigenvalue of $\bA\otimes \bB$ corresponding to the eigenvector $\bx\otimes \by$.
\end{lemma}

We introduce the associative and distributive properties of Kronecker products without a proof. Detailed proofs can be found in \citet{horn1994topics}.
\begin{remark}[Properties of Kronecker Products]
The Kronecker product is associative, i.e.,
$$
(\bA\otimes \bB) \otimes \bC =  \bA\otimes (\bB\otimes \bC).
$$
The Kronecker product is right–distributive, i.e.,
$$
(\bA+ \bB) \otimes \bC = \bA \otimes \bC + \bB \otimes \bC.
$$
The Kronecker product is left–distributive, i.e.,
$$
\bA\otimes (\bB+ \bC)  = \bA \otimes \bB + \bA \otimes \bC.
$$
Taking the transpose before carrying out the Kronecker product yields the same result as doing so afterwards, i.e.,
\begin{equation}\label{equation:kro_trans_pro}
(\bA\otimes \bB)^\top=\bA^\top \otimes \bB^\top.
\end{equation}
Taking the inverse before carrying out the Kronecker product yields the same result as doing so afterwards, i.e.,
$$
(\bA\otimes \bB)^{-1}=\bA^{-1} \otimes \bB^{-1}.
$$
\end{remark}
Using equality~\eqref{equation:kro_trans_pro}, we can therefore prove Equation~\eqref{equation:kron_pro1_trans} from Equation~\eqref{equation:kron_pro1_notrans}, or conversely.

\subsection{Khatri-Rao Product}
The Khatri-Rao product, often denoted by $\odot$, is a specialized matrix product that is closely related to the Kronecker product but is typically used in settings where the matrices involved are column-wise partitioned. 
\begin{definition}[Khatri-Rao Product]\label{definition:khatri-rao-product}
	The \textit{Khatri-Rao product} of two matrices  $\bA \in \real^{I\times K}$ and $\bB\in \real^{J\times K}$,  denoted by $\bA\odot \bB$, is defined as follows:
	$$
	\bA\odot\bB =
	\begin{bmatrix}
		\ba_1\otimes \bb_1 & \ba_2\otimes \bb_2 & \ldots & \ba_K\otimes \bb_K
	\end{bmatrix},
	$$
	which is a matrix of size $(IJ)\times K$. And it is known as the ``matching column-wise" Kronecker product.
	
\end{definition}

Consider partitioned matrices, the Khatri-Rao product has its partition-wise counterpart, denoted by $\odot_b$.
\begin{definition}[Partition-wise Khatri-Rao Product]\label{definition:partition_khatri_prod}
	The partition-wise Khatri-Rao product of two matrices  
	$$\bA=[\underbrace{\bA_1}_{\real^{I\times M_1}}, \underbrace{\bA_2}_{\real^{I\times M_2}}, \ldots, \underbrace{\bA_R}_{\real^{I\times M_R}}]
	\gap \text{and}\gap 
	\bB=[\underbrace{\bB_1}_{\real^{J\times N_1}}, \underbrace{\bB_2}_{\real^{J\times N_2}}, \ldots, \underbrace{\bB_R}_{\real^{J\times N_R}}],$$  
	denoted by $\bA\odot_b \bB$, is defined as follows:
	$$
	\bA\odot_b\bB =
	\begin{bmatrix}
		\bA_1\otimes \bB_1 & \bA_2\otimes \bB_2 & \ldots & \bA_R\otimes \bB_R
	\end{bmatrix},
	$$
	which is a matrix of size $(IJ)\times (\sum_{i=1}^{R} M_iN_i)$.
\end{definition}

The Khatri-Rao product shares some properties with the Kronecker product, but it is specifically designed for operations involving matrices with a common number of columns.
Based on the  definition of the Khatri-Rao product, for two vectors $\ba$ and $\bb$, we can observe that the Khatri-Rao product of these two vectors is equivalent to their Kronecker product:
$$
\ba \odot \bb = \ba\otimes\bb.
$$
Given three matrices $\bA, \bB$, and $\bC$, the ``distributive law" for the Khatri-Rao product follows that 
$$
\bA\odot \bB\odot \bC = (\bA\odot \bB)\odot \bC = \bA\odot (\bB\odot \bC).
$$
Additionally, when $\bA, \bB\in\real^{I\times K}$ are semi-orthogonal matrices~\footnote{Given a  matrix $\bQ\in\real^{m\times n}$ with mutually orthonormal columns and $m\neq n$, semi-orthogonal matrices satisfying $\bQ^\top\bQ=\bI \in \real^{n\times n}$. While orthogonal matrices has the property that $\bQ\bQ^\top=\bQ^\top\bQ=\bI$ and $m=n$.}, the Khatri-Rao product $\bA\odot \bB$ is also semi-orthogonal.

\subsection{More Properties of Special Matrix Products}
\paragraph{\colorbox{\mdframecolorBox}{$(\bA\odot \bB )^\top (\bA\odot \bB )= 
		(\bA^\top\bA) \hadaprod (\bB^\top\bB)$}.}
Moreover, we observe that for two matrices   $\bA \in \real^{I\times K}$ and $\bB\in \real^{J\times K}$, the following relationship holds:
\begin{equation}
	\bZ = (\bA\odot \bB )^\top (\bA\odot \bB )=
	\begin{bmatrix}
		(\ba_1\otimes \bb_1)^\top  \\
		(\ba_2\otimes \bb_2)^\top  \\
		\vdots \\ 
		(\ba_K\otimes \bb_K)^\top 
	\end{bmatrix}
	\begin{bmatrix}
		\ba_1\otimes \bb_1 & \ba_2\otimes \bb_2 & \ldots & \ba_K\otimes \bb_K
	\end{bmatrix},
\end{equation}
where $\bZ\in \real^{K\times K}$, and each entry $(i,j)$, denoted by $z_{ij}$, is given by 
$$
z_{ij} = (\ba_i\otimes \bb_i)^\top (\ba_j\otimes \bb_j) = (\ba_i^\top\ba_j)(\bb_i^\top\bb_j),
$$
where the last equality is derived from Equation~\eqref{equation:kronecker-vector-find2}. Therefore, $\bZ$ can be expressed equivalently  as 
\begin{equation}\label{equation:two-khatri-rao-pro-equi}
	\bZ = (\bA\odot \bB )^\top (\bA\odot \bB ) = 
	(\bA^\top\bA) \hadaprod (\bB^\top\bB).
\end{equation}

\paragraph{\colorbox{\mdframecolorBox}{$(\bA\odot \bB )^\top (\bC\odot \bD )= 
		(\bA^\top\bC) \hadaprod (\bB^\top\bD)$}.}
Similarly, given  $\bA,\bC \in \real^{I\times K}$ and $\bB, \bD\in \real^{J\times K}$, it follows that 
\begin{equation}
	(\bA\odot \bB )^\top (\bC\odot \bD ) = 
	(\bA^\top\bC) \hadaprod (\bB^\top\bD).
\end{equation}

\paragraph{\colorbox{\mdframecolorBox}{$(\bA\otimes\bB)(\bC\odot \bD)=(\bA\bC)\odot(\bB\bD)$}.}
To see this, given $\bA\in\real^{I\times J}$, $\bB\in\real^{K\times L}$, $\bC\in\real^{J\times P}$, and $\bD\in\real^{L\times P}$, then  $(\bC\odot \bD)$ has a shape of ${JL\times P}$. Each column of $(\bC\odot \bD)$ is the Kronecker product of the corresponding columns of $\bC$ and $\bD$. Specifically, the $p$-th column of $(\bC\odot \bD)$ is $\bc_p\otimes \bd_p$, $p\in\{1,2,\ldots, P\}$. 
Therefore, the $p$-th column of the left-hand side, by Equation~\eqref{equation:kronecker_eq3}, is
$$
(\bA\otimes\bB)(\bc_p\otimes \bd_p)=(\bA\bc_p)\otimes(\bB \bd_p).
$$
For the right-hand side, the $p$-th columns of the matrices $(\bA\bC)$ and $(\bB\bD)$ are $\bA\bc_p$ and $\bB\bd_p$, respectively. Hence, the $p$-th column of the Khatri-Rao product  $(\bA\bC)\odot(\bB\bD)$ is 
$$
(\bA\bc_p)\otimes (\bB\bd_p).
$$
Since the $p$-th columns of both sides of the equation are identical for all  $p\in\{1,2,\ldots, P\}$, this concludes the proof that:
$$
(\bA\otimes\bB)(\bC\odot \bD)=(\bA\bC)\odot(\bB\bD).
$$

To conclude, it follows that
\begin{align}
	&	(\ba\otimes \bb)^\top(\bc\otimes \bd)&=&(\ba^\top\bc)(\bb^\top\bd);\\
	&	(\ba\otimes \bb)^\top(\ba\otimes \bb) &=&\norm{\ba}^2 \cdot \norm{\bb}^2;\\
	&	(\bA\otimes \bB)^\top (\bC\otimes \bD) &=& (\bA^\top\bC) \otimes (\bB^\top\bD), \gap \left(\text{\parbox{11em}{with $\bA,\bC$ same shape, \\$\bB,\bD$ same shape}}\right);\\
	&(\bA\otimes \bB) (\bC\otimes \bD) &=& (\bA\bC) \otimes (\bB\bD);\label{equ:mat_kro_prod}\\
	&(\bA\otimes \bB) (\bc \otimes \bd) &=& (\bA\bc) \otimes (\bB\bd);\\
	& (\bA\otimes \bB)^+ &=& \bA^+ \otimes \bB^+;\\
	&	\bA\odot \bB\odot \bC &=& (\bA\odot \bB)\odot \bC \\
	&\gap	&=& \bA\odot (\bB\odot \bC);\\
	&	(\bA\odot \bB )^\top (\bA\odot \bB ) &=& 	(\bA^\top\bA) \hadaprod (\bB^\top\bB);\label{equ:aa_hada_bb}\\
	&(\bA\odot \bB )^\top (\bC\odot \bD ) &=&  (\bA^\top\bC) \hadaprod (\bB^\top\bD);\label{equ:ab_top_cd}\\
	&(\bA\otimes\bB)(\bC\odot \bD) &=& (\bA\bC)\odot(\bB\bD).\label{equ:atimesb_codotb}
\end{align}

We conclude this section by describing that the ranks for Hadamard products and Kronecker products are ``multiplicative", and the rank for Khatri-Rao products is nondecreasing.
\begin{theorem}[Rank of Hadamard Products, \citep{horn1990hadamard}]\label{theorem:rank_hada_prod}
	If $\bA_1,\bA_2\in\real^{m\times n}$ are rank-$r_1$ and rank-$r_2$, respectively,
	then their Hadamard product $\bA_1\hadaprod\bA_2$ has rank at most $r_1\cdot r_2$.
\end{theorem}
\begin{proof}[of Theorem~\ref{theorem:rank_hada_prod}]
	For any rank-$r_1$ matrix $\bA_1$ and rank-$r_2$ matrix $\bA_2$, it can be expressed as sum of rank-$1$ matrices: 
	$$
	\begin{aligned}
		\bA_1&=\bC_1\bD_1^\top = \sum_{i=1}^{r_1} \bc_{1i}\bd_{1i}^\top;\\
		\bA_2&=\bC_2\bD_2^\top = \sum_{j=1}^{r_2} \bc_{2j}\bd_{2j}^\top,
	\end{aligned}
	$$
	where $\bC_1\in\real^{m\times r_1}$ and $\bD_1\in\real^{n\times r_1}$ are rank-$r_1$ matrices,  $\bC_2\in\real^{m\times r_2}$ and $\bD_2\in\real^{n\times r_2}$ are rank-$r_2$ matrices, and $\bc_{1i}, \bd_{1i}$ are the $i$-th column of $\bC_1, \bD_1$ (same for $\bc_{2j},\bd_{2j}$, see Figure~\ref{fig:hada_rank}).
	Therefore, 
	$$
	\bA_1\hadaprod\bA_2 = \left(\sum_{i=1}^{r_1} \bc_{1i}\bd_{1i}^\top\right)\hadaprod  \left(\sum_{j=1}^{r_2} \bc_{2j}\bd_{2j}^\top\right)
	=\sum_{i=1}^{r_1}\sum_{j=1}^{r_2} \left(\bc_{1i}\bd_{1i}^\top \right) \hadaprod \left(\bc_{2j}\bd_{2j}^\top\right).
	$$
	The Hadamard product is the sum of $r_1\cdot r_2$ rank-1 matrices, and thus has rank at most $r_1\cdot r_2$.
\end{proof}

\begin{figure}[h]
	\centering
	\includegraphics[width=1\textwidth]{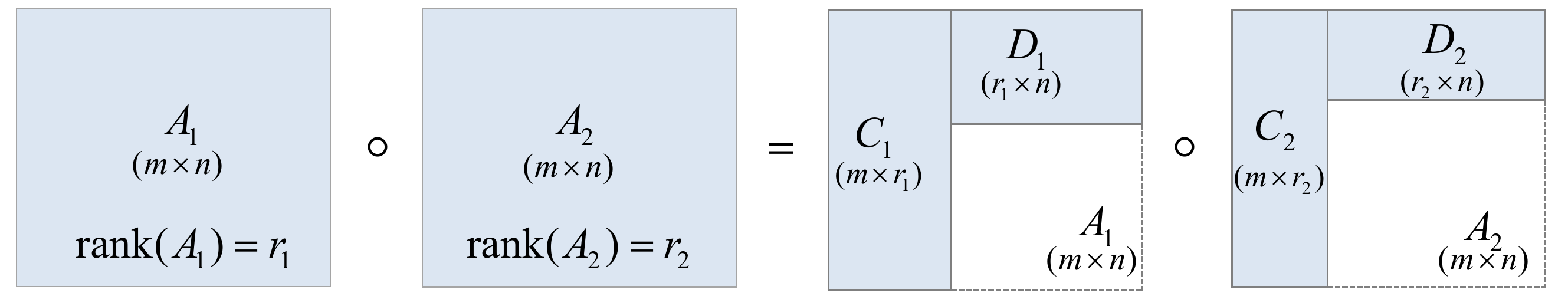}
	\caption{Diagram illustrating the rank in a Hadamard product.}
	\label{fig:hada_rank}
\end{figure}

\begin{theorem}[Rank, Trace of Kronecker Products]\label{theorem:rank_kronec_prod}
	Given matrices $\bA \in \real^{I\times J}$ and $\bB\in \real^{K\times L}$, then 
	$$
	\rank(\bA\otimes \bB) = \rank(\bA)\rank(\bB).
	$$
	That is, real rank is multiplicative under Kronecker product.
	When $\bA$ and $\bB$ are square and symmetric, then we also have 
	$$
	\trace(\bA\otimes \bB) = \trace(\bA)\trace(\bB)
	$$
\end{theorem}
\begin{proof}[of Theorem~\ref{theorem:rank_kronec_prod}]
Suppose $\bA$ and $\bB$ admit full SVD $\bU_1^\top\bA\bV_1=\bSigma_1$ and $\bU_2^\top\bB\bV_2=\bSigma_2$, respectively. 
Applying Equality~\eqref{equ:mat_kro_prod}, we have 
$$
\begin{aligned}
\bSigma_1\otimes \bSigma_2 &= (\bU_1^\top\bA\bV_1)\otimes (\bU_2^\top\bB\bV_2)\\
&= (\bU_1^\top\otimes \bU_2^\top) (\bA\bV_1\otimes \bB\bV_2)
=(\bU_1^\top\otimes \bU_2^\top)(\bA\otimes \bB)(\bV_1\otimes \bV_2).
\end{aligned}
$$
Since the Kronecker product of orthogonal matrices is also orthogonal (Lemma~\ref{lemma:krokecker_keep_special}), $(\bSigma_1\otimes \bSigma_2)$ and $(\bA\otimes \bB)$ are similar matrices.  We have, by the fact that similar matrices have the same rank,
$$
\rank(\bSigma_1\otimes \bSigma_2) = \rank(\bA\otimes \bB),
$$
where $\bSigma_1\otimes \bSigma_2$ has $\rank(\bSigma_1)\rank(\bSigma_2)$ nonzero entries, indicating that $\rank(\bSigma_1\otimes \bSigma_2)=\rank(\bSigma_1)\rank(\bSigma_2)$.

For the second part, since $\bA,\bB$ are symmetric, $\bSigma_1\otimes \bSigma_2$ and $\bA\otimes \bB$ are similar matrices, indicating their traces are the same~\footnote{The trace and rank of $\bA$ are equal to those of matrix $\bP\bA\bP^{-1}$ for any nonsingular matrix $\bP$. And the trace of a symmetric matrix is the sum of its eigenvalues (singular values).}.
This completes the proof.
\end{proof}

\begin{theorem}[Rank of Khatri-Rao Products]\label{theorem:rank_khatri_prod}
Given matrices $\bA \in \real^{I\times K}$ and $\bB\in \real^{J\times K}$, then 
$$
\rank(\bA\odot  \bB) \geq  \max\{\rank(\bA),\rank(\bB)\}.
$$
\end{theorem}
\begin{proof}[of Theorem~\ref{theorem:rank_khatri_prod}]
For any matrix $\bC$, we can find a nonsingular $\bS$ such that there exists a row of $\bS\bC$ with all nonzeros.
Therefore, assume there exist nonsingular matrices $\bP$ and $\bT$ such that $\bP\bA$ and $\bT\bB$ have at least one row containing all nonzero elements.
Since $\rank(\bP\bA)=\rank(\bA)$, $\rank(\bT\bB)=\rank(\bB)$, and $(\bP\bA)\odot (\bT\bB) = (\bP\otimes \bT)(\bA\odot \bB)$ with $\rank(\bA\odot \bB) = \rank((\bP\otimes \bT)(\bA\odot \bB))$ (Lemma~\ref{lemma:left_mul_krank}, Equation~\eqref{equ:atimesb_codotb}, and $(\bP\otimes \bT)$ is also nonsingular from Lemma~\ref{lemma:krokecker_keep_special}), it suffices to show that $\rank(\bP\bA\odot \bT\bB) \geq \max\{\rank(\bP\bA),\rank(\bT\bB)\}$.

Since $\bT\bB$ contains a row that has all nonzero elements, then $(\bP\bA\odot \bT\bB)$ contains a submatrix equal to $\bP\bA$, rescaled columnwise by the elements of that row of $\bT\bB$. Therefore, $\rank(\bP\bA\odot \bT\bB) \geq \rank(\bP\bA)$. Similarly, we can prove $\rank(\bP\bA\odot \bT\bB) \geq \rank(\bT\bB)$. This completes the proof.
\end{proof}

In many cases, we may also consider a special rank known as  $k$-Rank or Kruskal rank, which captures specific structural dimensions of matrices.  $k$-ranks appear in the formulation
of the famous Kruskal condition for the uniqueness of  CANDECOMP-PARAFAC decomposition \citep{carroll1970analysis,harshman1970foundations, de2008decompositionspart}.
\begin{definition}[$k$-Rank or Kruskal Rank]
	The Kruskal rank or $k$-rank of a matrix $\bA$, denoted by $\rank_k(\bA)$ or $k_{\bA}$, is the largest  number $r$ such that any set of $r$ columns of $\bA$ is linearly independent.
\end{definition}
When the matrix has a $k$-rank of $\sigma$, this means that no column in any subset of size $\sigma$ can be expressed as a linear combination of the others in that subset.
Apparently, a rank-$r$ matrix can have a $k$-rank of 1 if there are two identical columns in the matrix. 
Thus, the  $k$-rank provides insight into a specific type of structural redundancy within the matrix.
Specifically, 
\begin{itemize}
	\item When the matrix has full column rank, then adding columns to the matrix may increase or decrease the $k$-rank.
	\item When the matrix does not have full column rank, then adding columns to the matrix can never increase the $k$-rank.
\end{itemize}

\begin{lemma}[Left-Multiplying with Nonsingular]\label{lemma:left_mul_krank}
Left-multiplying a matrix by a nonsingular matrix can be understood as applying the same nonsingular transformation to each column of the original matrix. This transformation preserves the linear independence of the columns; thus, the linear dependence relationships among the columns remain unchanged.
Therefore, left-multiplying a matrix with a nonsingular matrix will not change either the rank or $k$-rank of the resulting matrix.
\end{lemma}

\begin{theorem}[$k$-Rank of Khatri-Rao Products]\label{theorem:kkrank_khatri_prod}
	Given matrices $\bA \in \real^{I\times K}$ and $\bB\in \real^{J\times K}$, then 
	$$
	\rank_k(\bA\odot  \bB) \geq  \min\{\rank_k(\bA)+\rank_k(\bB)-1, K\}.
	$$
\end{theorem}
\begin{proof}[of Theorem~\ref{theorem:kkrank_khatri_prod}]
The proof is based on \citet{ten2000k}.
Without loss of generality, we assume $k$-rank of $\bA\odot \bB$ is less than $K$; otherwise the proof is trivial.
Let $S$ be the smallest number of linearly dependent columns of $\bA\odot \bB$. Therefore, $\rank_k(\bA\odot  \bB)  = S-1$.
We collect some set of $S$ linearly dependent columns of $\bA\odot  \bB$ into $\bA_S\odot \bB_S$, where $\bA_S$ and $\bB_S$ contain the corresponding columns of $\bA$ and $\bB$, respectively. 
Therefore, there exists a vector $\bc_S$ with nonzero elements satisfying $(\bA_S\odot \bB_S)\bc_S=\bzero$; otherwise we can reduce the number $S$. 
This implies $\bA_S\bC_S\bB_S^\top=\bzero $, where $\bC_S=\diag(\bc_S)$ is nonsingular.
Therefore, we have 
$$
0=\rank(\bA_S\bC_S\bB_S^\top)\geq \rank(\bA_S)+\rank(\bB_S)-S.~\footnote{Follows from the Frobenius inequality: $\rank(\bA\bB\bC)\geq \rank(\bA\bB)+\rank(\bB\bC)-\rank(\bB)$.}
$$
Since $\rank(\bA_S)\geq \rank_k(\bA_S)\geq \rank_k(\bA)$ and $\rank(\bB_S)\geq \rank_k(\bB_S)\geq \rank_k(\bB)$, we complete the proof.
\end{proof}

\section{Low-Rank Hadamard Decomposition}\label{section:low_rank_hadamard}

Ws further consider the Hadamard decomposition of a matrix $\bA$ such that $\bA$ can be represented as the Hadamard product of two low-rank matrices: $\bA=\bA_1\hadaprod \bA_2$.
\paragraph{Non-Factorizability  Issue.}
Although when $\bA_1\in\real^{n^2\times n^2}$ and $\bA_2\in\real^{n^2\times n^2}$ share the same rank $n$, the Hadamard product $\bA_1\hadaprod \bA_2$ can have a maximum rank of $n^2$, not any arbitrary matrix $\bA\in\real^{n^2\times n^2}$ of rank-$n^2$  may have a representation as the Hadamard product of two lower-rank matrices:
\begin{itemize}
	\item The Hadamard decomposition $\bA = \bA_1 \hadaprod  \bA_2$, where $\bA_1$ and $\bA_2$ are rank-$n$ factors, encodes a system of nonlinear equations.
	\item This system includes $n^2 \times n^2 = n^4$ equations (one per entry of $ \bA$) and, due to the low-rank constraint on the two Hadamard factors $\bA_1$ and $ \bA_2$, only $(n^2 n + n n^2) = 2n^3$ variables exist.
	\item For $n > 2$, there are more equations than variables, suggesting that all the equations will be simultaneously satisfied only in special cases.
	For example, if the matrix $\bA$ includes a row or a column with all but a single entry being 0, then not all the equations in the system can be satisfied \citep{ciaperoni2024hadamard}.
	
\end{itemize}

Therefore, we focus on solving the low-rank reconstruction problem for the Hadamard decomposition.
In Theorem~\ref{theorem:rank_hada_prod}, assuming that  $\bA_1$ and $\bA_2$ have the same rank $K$, our goal is to reconstruct the design matrix $\bA$ through the Hadamard product $\bA_1\hadaprod\bA_2$. 
Building upon the matrix factorization approach used in ALS,
we now concentrate on algorithms for solving the \textit{low-rank Hadamard decomposition} problem  (we may refer to low-rank Hadamard decomposition simply as Hadamard decomposition for brevity when the context is clear; the same applies to the low-rank Kronecker and Khatri-Rao decompositions discussed later):
\begin{itemize}
	\item Given a real matrix $\bA\in \real^{M\times N}$, find  matrix factors $\bA_1\in \real^{M\times N}$ and $\bA_2\in \real^{M\times N}$ such that: 
	$$
	\min\,\,L(\bC_1, \bD_1, \bC_2,\bD_2) = \norm{\bA_1\hadaprod \bA_2-\bA}_F^2= \norm{(\bC_1\bD_1)\hadaprod (\bC_2\bD_2) -\bA}_F^2,
	$$
	where $\bC_1, \bC_2\in\real^{M\times K}$, and $\bD_1, \bD_2\in\real^{K\times N}$.~\footnote{In the proof of Lemma~\ref{theorem:rank_hada_prod}, we consider the matrices $\bD_1, \bD_2\in\real^{N\times K}$. To abuse notations, we use $\bD_1, \bD_2\in\real^{K\times N}$ here for the ease of deriving gradients in the sequel.}
\end{itemize}
The low-rank decomposition is generally necessary because many natural phenomena exhibit multiplicative or conjunctive relationships \citep{ciaperoni2024hadamard}.
For instance, consider a study on risk factors for a disease with two predictors: smoking status (yes/no) and alcohol consumption (yes/no). The multiplicative model would account for not only the main effects of smoking and alcohol consumption but also their interaction.
The (low-rank) Hadamard decomposition offers an alternative approach.

Following the alternating descent framework, in each iteration, the matrices $\bC_1, \bD_1, \bC_2$, and $\bD_2$ are updated sequentially by taking a step in the direction opposite to the gradient of the objective function.
It then can be shown that 
$$
\nabla L(\bC_1)=\nabla L(\bC_1|\bD_1, \bC_2,\bD_2)=2\left(\left((\bC_1\bD_1)\hadaprod (\bC_2\bD_2) -\bA\right)\hadaprod (\bC_2\bD_2)\right)\bD_1.
$$
\begin{proof}
	For brevity, we show the gradient of $\bA$ for $f(\bA) = \norm{\bA\bB\hadaprod \bC - \bD}_F^2$.
	We have 
	$$
	\begin{aligned}
		f(\bA)&=\norm{\bA\bB\hadaprod \bC - \bD}_F^2 
		= \trace\left((\bA\bB\hadaprod \bC - \bD)^\top(\bA\bB\hadaprod \bC - \bD) \right)\\
		&= \trace\left((\bA\bB\hadaprod \bC)^\top (\bA\bB\hadaprod \bC)\right)-2\trace\left((\bA\bB\hadaprod \bC)^\top\bD\right) + \trace(\bD^\top\bD).
	\end{aligned}
	$$
	Consider the first term, we have 
	$$
	\frac{\partial \trace\left((\bA\bB\hadaprod \bC)^\top (\bA\bB\hadaprod \bC)\right)}{\partial \bA}
	=2(\bA\bB)\hadaprod \bC\hadaprod \bC \cdot \bB^\top.
	~\footnote{Use the fact that $\frac{\partial \trace\left((\bE\hadaprod \bC)^\top(\bE\hadaprod \bC) \right)}{\partial \bE}=2\bE\hadaprod \bC\hadaprod \bC$, which can be derived element-wise.}
	$$
	For the second term, it follows that 
	$$
	-2\frac{\partial \trace\left((\bA\bB\hadaprod \bC)^\top\bD\right)}{\partial \bA}
	=-2\bD\hadaprod \bC \cdot \frac{\partial \bA\bB}{\partial \bA}
	=-2\bD\hadaprod \bC \cdot\bB^\top.
	~\footnote{Use the fact that $\frac{\partial\trace( (\bA\hadaprod \bC)^\top\bD )}{\partial \bA} = \bD\hadaprod \bC$, which can be derived element-wise. Since $\trace( (\bA\hadaprod \bC)^\top\bD)=\sum_{i,j} d_{ij}a_{ij}c_{ij}$ and thus $\frac{\partial \trace( (\bA\hadaprod \bC)^\top\bD)}{\partial a_{ij}}=d_{ij}c_{ij}$.}
	$$
	The third term is a constant w.r.t. to $\bA$. 
	Therefore, $\frac{\partial f(\bA)}{\bA} = 2(\bA\bB)\hadaprod \bC\hadaprod \bC \cdot \bB^\top-2\bD\hadaprod \bC \cdot\bB^\top.$
	Substituting with $\bA:=\bC_1$, $\bB:=\bD_1$, $\bC:=\bC_2\bD_2$, and $\bD:=\bA$ completes the proof.
\end{proof}
The gradients with respect to $\bD_1, \bC_2$, and $\bD_2$ can be derived analogously.
Therefore, the alternating descent method can be described by Algorithm~\ref{alg:ad_hadamad_svd} for obtaining the low-rank approximation of Hadamard decomposition.
We may observe that, unlike the ALS algorithm, the low-rank Hadamard decomposition does not admit a closed-form solution.

\begin{algorithm}[h] 
\caption{Alternating Descent with Gradient Descent for Hadamard Decomposition: A regularization can also be added into the gradient descent update (see Section~\ref{section:regularization-extention-general}).}
\label{alg:ad_hadamad_svd}
\begin{algorithmic}[1] 
\Require Matrix $\bA\in \real^{M\times N}$;
\State Initialize $\bC_1,\bC_2\in \real^{M\times K}$, and $\bD_1,\bD_2\in \real^{K\times N}$; 
\State Choose a stop criterion on the approximation error $\delta$;
\State Choose  step size $\eta$;
\State Choose the maximal number of iterations $C$;
\State $iter=0$; \Comment{Count for the number of iterations}
\While{$\norm{(\bC_1\bD_1)\hadaprod (\bC_2\bD_2) -\bA}_F^2>\delta $ and $iter<C$}
\State $iter=iter+1$; 
\State $\Delta = \left((\bC_1\bD_1)\hadaprod (\bC_2\bD_2) -\bA\right)$;
\State $\bC_1 = \bC_1-\eta \nabla L(\bC_1)=\bC_1-\eta\cdot 2\left(\Delta \hadaprod (\bC_2\bD_2)\right)\bD_1$;
\State $\bD_1 = \bD_1-\eta \nabla L(\bD_1)=\bD_1-\eta\cdot 2 \left\{\left(\Delta^\top \hadaprod (\bC_2\bD_2)^\top\right)\bC_1\right\}^\top$;
\State $\bC_2 = \bC_2-\eta \nabla L(\bC_2)=\bC_2 - \eta \cdot 2 \left(\Delta \hadaprod (\bC_1\bD_1)\right)\bD_2$;
\State $\bD_2 = \bD_2-\eta \nabla L(\bD_2)=\bD_2-\eta \cdot 2 \left\{\left( \Delta^\top \hadaprod (\bC_1\bD_1)^\top \right)\bC_2\right\}^\top$;

\EndWhile
\State Output $\bC_1,\bD_1, \bC_2,\bD_2$;
\end{algorithmic} 
\end{algorithm}

\subsection{Rank-One Update }
Following the rank-one update of ALS (Section~\ref{section:alt-columb-by-column}), we consider to update the $n$-th column $\bd_{1,n}$ of $\bD_1$, for $n\in\{1,2,\ldots, N\}$.
Analogously, we can obtain the gradient of $\bd_{1,n}$:
\begin{equation}\label{equation:hada_rkone1}
	\begin{aligned}
		\nabla L(\bd_{1,n}) 
		=\frac{\partial L(\bd_{1,n})}{\partial \bd_{1,n}}
		&= 2\bC_1^\top \left((\bC_1\bd_{1,n}) \hadaprod \ba_{2,n} \hadaprod \ba_{2,n}\right) - 2\bC_1^\top (\ba_n \hadaprod \ba_{2,n})\\
		&=2\bC_1^\top \left(\left[(\bC_1\bd_{1,n}) \hadaprod \ba_{2,n}-\ba_n \right] \hadaprod\ba_{2,n}\right), \gap n\in\{1,2,\ldots, N\}, 
	\end{aligned}
\end{equation}
where $\ba_{2,n}$ is the $n$-th column of $\bA_2=\bC_2\bD_2$, $n\in\{1,2,\ldots, N\}$. The gradient of the columns of $\bD_2$ can be calculated similarly.

Suppose further $\bC_1^\top=[\bc_{1,1}, \bc_{1,2}, \ldots, \bc_{1,M}]\in\real^{K\times M}$,  $\bB=\bA^\top=[\bb_1, \bb_2, \ldots, \bb_M]\in\real^{N\times M}$, and $\bB_2=\bA_2^\top=(\bC_2\bD_2)^\top=[\bb_{2,1}, \bb_{2,2}, \ldots, \bb_{2,M}]\in\real^{N\times M}$, i.e., the row partitions of $\bC_1$, $\bA$, and $\bA_2=(\bC_2\bD_2)$, respectively.
Then, the gradient of $\bc_{1,m}$ is 
\begin{equation}\label{equation:hada_rkone2}
	\nabla L(\bc_{1,m}) 
	=\frac{\partial L(\bc_{1,m})}{\partial \bc_{1,m}}
	= 2\bD_1 \left(\left[(\bD_1^\top\bc_{1,m}) \hadaprod \bb_{2,m}-\bb_m \right] \hadaprod\bb_{2,m}\right), \gap m\in\{1,2,\ldots, M\}.
\end{equation}
The gradient of the rows of $\bC_2$ can be obtained analogously.
Therefore,  Algorithm~\ref{alg:ad_hadamad_svd} can be modified to update the columns of $\bD_1, \bD_2$ and the rows of $\bC_1, \bC_2$ iteratively (referred to as the rank-one update).

\subsection{Missing Entries}
The rank-one update can be extended to the Netflix context, in which case many entries of $\bA$ are missing: assume $\bA$ is a low-rank matrix, we aim to fill in the missing entries of matrix $\bA$.

Again, let $\bo_n\in \{0,1\}^M, n\in\{1,2,\ldots, N\}$ represent the movies rated by user $n$, where $o_{nm}=1$ if user $n$ has rated movie $m$, and $o_{nm}=0$ otherwise.
Similarly, let $\bp_m \in\{0,1\}^{N}, m\in\{1,2,\ldots,M\}$ denotes the users who have rated  movie $m$, with $p_{mn}=1$ if the movie $m$ has been rated by user $n$, and $p_{mn}=0$ otherwise.
Then, Equation~\eqref{equation:hada_rkone1} and~\eqref{equation:hada_rkone2} become
\begin{align}
	\nabla L(\bd_{1,n}) 
	&=2\bC_1[\bo_n,:]^\top \left(\left[(\bC_1[\bo_n,:]\bd_{1,n}) \hadaprod \ba_{2,n}[\bo_n]-\ba_n[\bo_n] \right] \hadaprod\ba_{2,n}[\bo_n]\right),\nonumber \\
	&\gap\gap\gap\gap\gap\gap\gap\gap\gap\gap\gap\gap\gap\gap\gap\gap n\in\{1,2,\ldots, N\};\label{equation:hada_rkone3} \\
	\nabla L(\bc_{1,m}) 
	&= 2\bD_1[:,\bp_m] \left(\left[(\bD_1[:,\bp_m]^\top\bc_{1,m}) \hadaprod \bb_{2,m}[\bp_m]-\bb_m[\bp_m] \right] \hadaprod\bb_{2,m}[\bp_m]\right),\nonumber\\
	&\gap\gap\gap\gap\gap\gap\gap\gap\gap\gap\gap\gap\gap\gap\gap\gap  m\in\{1,2,\ldots, M\} \label{equation:hada_rkone4}.
\end{align}
Since the Hadamard product commutes, due to symmetry, the gradient of $L(\bd_{2,n}), n\{1,2,\ldots, N\}$ and $\bc_{2,m},  m\in\{1,2,\ldots, M\}$ can be obtained similarly. 
The process for predicting the missing entries in $\bA$ is then formulated in Algorithm~\ref{alg:ad_hadamad_missen}.

\begin{algorithm}[h] 
\caption{Alternating Descent with Gradient Descent for Hadamard Decomposition with Missing Entries: A regularization can also be added into the gradient descent update (see Section~\ref{section:regularization-extention-general}).}
\label{alg:ad_hadamad_missen}
\begin{algorithmic}[1] 
\Require Matrix $\bA\in \real^{M\times N}$;
\State Initialize $\bC_1,\bC_2\in \real^{M\times K}$, and $\bD_1,\bD_2\in \real^{K\times N}$; 
\State Choose a stop criterion on the approximation error $\delta$;
\State Choose  step size $\eta$;
\State Choose the maximal number of iterations $C$;
\State $iter=0$; \Comment{Count for the number of iterations}
\While{$\norm{(\bC_1\bD_1)\hadaprod (\bC_2\bD_2) -\bA}_F^2>\delta $ and $iter<C$}
\State $iter=iter+1$; 
\For{$n=1,2,\ldots, N$}
\State $\bd_{1,n}=\bd_{1,n}-\eta \nabla L(\bd_{1,n}) $;  \Comment{Equation~\eqref{equation:hada_rkone3}}
\State $\bd_{2,n}=\bd_{2,n}-\eta \nabla L(\bd_{2,n}) $;
\EndFor
\For{$m=1,2,\ldots, M$}
\State $\bc_{1,m}=\bc_{1,m}-\eta \nabla L(\bc_{1,m})  $; \Comment{Equation~\eqref{equation:hada_rkone4}}
\State $\bc_{2,m}=\bc_{2,m}-\eta \nabla L(\bc_{2,m})  $;
\EndFor
\EndWhile
\State Output $\bC_1,\bD_1, \bC_2,\bD_2$;
\end{algorithmic} 
\end{algorithm}

\section{Low-Rank Kronecker Decomposition}\label{section:low_rank_kronecker}

Similarly to the low-rank Hadamard decomposition, we can also consider the low-rank Kronecker decomposition.
Suppose the design matrix $\bA\in\real^{M\times N}$ has dimensions such that $M=m_1m_2$ and $N=n_1n_2$. 
Consider two matrices $\bB\in\real^{m_1\times n_1}$ and $\bC\in\real^{m_2\times n_2}$ such that $\bB\otimes \bC\in\real^{M\times N}$.
Theorem~\ref{theorem:rank_kronec_prod} shows that $\rank(\bB\otimes \bC)=\rank(\bB)\rank(\bC)$.
The goal then becomes 
\begin{itemize}
	\item Given a real matrix $\bA\in \real^{M\times N}$, find  matrix factors $\bB\in\real^{m_1\times n_1}$ and $\bC\in\real^{m_2\times n_2}$ such that: 
	$$
	\min\,\,L(\bB, \bC) = \norm{\bB\otimes \bC-\bA}_F^2.
	$$
\end{itemize}

Given the definition of the Kronecker product $\bB\otimes \bC$ (Definition~\ref{definition:kronecker-product}), we may consider the \textit{uniform blocking} of matrix $\bA$:
$$
\bA=
\begin{bmatrix}
	\bA_{11} &\bA_{12} & \ldots &\bA_{1,n_1} \\
	\bA_{21} &\bA_{22} & \ldots &\bA_{2,n_1} \\
	\vdots & \vdots & \ddots & \vdots \\
	\bA_{m_1,1} &\bA_{m_1,2} & \ldots &\bA_{m_1,n_1} \\
\end{bmatrix}, 
\gap 
\bA_{ij}\in\real^{m_2\times n_2},
$$
where the $(i,j)$-th block is an $m_2\times n_2$ matrix with $\bA_{ij}=\bA[(i-1)m_2+1:i\cdot m_2, (j-1)n_2+1:j\cdot n_2]$.
While the $(i,j)$-th block of $\bB\otimes \bC$ is $b_{ij}\bC$. 
Therefore, when $\bC$ is held constant,  the objective function with respect to $\bB$ (or $b_{ij}$ for $i\in\{1,2,\ldots,m_1\}, j\in\{1,2,\ldots,n_1\}$) is 
$$
L(\bB) = \sum_{i=1}^{m_1}\sum_{j=1}^{n_1} \norm{\bA_{ij} - b_{ij}\bC}_F^2.
$$
Analogously, when keeping $\bB$ fixed, the objective function with respect to $\bC$ is 
$$
L(\bC) =\sum_{i=1}^{m_2}\sum_{j=1}^{n_2} \norm{\widetildebA_{ij}- c_{ij}\bB}_F^2,
$$
where $\widetildebA_{ij}=\bA[i:m_2:M, j:n_2:N]$, i.e., slicing the row every $m_2$ indices and the column every $n_2$ indices; the row indices of $\widetildebA_{ij}$ are $i, i+m_2, i+2m_2, \ldots$, and the column indices are $j, j+n_2, j+2n_2,\ldots$.

\subsection{Kronecker Decomposition via Alternating Least Squares}
Thinking at the block level for matrices can allow us to solve the least squares problem for each block independently.
When working with large matrices, especially those that can be naturally divided into smaller blocks, it is often beneficial to think at the block level. It enables us to determine least squares solutions for each block separately, which can greatly simplify the overall optimization process.
\begin{lemma}[Kronecker Decomposition via Least Squares \citep{van1993approximation}]\label{lemma:lrank_kro_opt}
	Suppose $\bA\in \real^{M\times N}$ with $M=m_1m_2$ and $N=n_1n_2$. If $\bC\in\real^{m_2\times n_2}$ is fixed, then the matrix $\bB$ defined by 
	\begin{equation}\label{equation:kro_ls_1}
		b_{ij} =\frac{\trace(\bA_{ij}^\top \bC)}{\trace(\bC^\top\bC)}, \gap i\in\{1,2,\ldots, m_1\}, j\in\{1,2,\ldots, n_1\}
		\footnote{The trace operator considers only the diagonal elements of a matrix. Therefore, in practice, it's more appropriate to compute the Hadamard product of the two matrices and then sum all the squared elements.}
	\end{equation}
	minimizes $\norm{\bB\otimes \bC-\bA}_F$, where $\bA_{ij}=\bA[(i-1)m_2+1:i\cdot m_2, (j-1)n_2+1:j\cdot n_2]$ is $(i,j)$-th block of $\bA$ with size $m_2\times n_2$.
	Analogously, if $\bB\in\real^{m_1\times n_1}$ is fixed, then the matrix $\bC$ defined by 
	\begin{equation}\label{equation:kro_ls_2}
		c_{ij} =\frac{\trace(\widetilde{\bA}_{ij}^\top\bB)}{\trace(\bB^\top\bB)}, \gap i\in\{1,2,\ldots, m_2\}, j\in\{1,2,\ldots, n_2\}
	\end{equation}
	minimizes $\norm{\bB\otimes \bC-\bA}_F$, where $\widetilde{\bA}_{ij}=\bA[i:m_2:M, j:n_2:N]$.
\end{lemma}
\begin{proof}[of Lemma~\ref{lemma:lrank_kro_opt}]
	Consider the  loss of the $(i,j)$-th block:
	$$
	\norm{\bA_{ij} - b_{ij}\bC}_F^2 = \trace\left((\bA_{ij} - b_{ij}\bC)^\top(\bA_{ij} - b_{ij}\bC)\right)
	=\norm{\bA_{ij}}_F^2 - 2b_{ij}\trace(\bC^\top\bA_{ij}) + b_{ij}^2 \norm{\bC}_F^2.
	$$
	Therefore, it follows that 
	$$
	\frac{\partial L(\bB, \bC)}{\partial b_{ij}} 
	= -2\trace(\bC^\top\bA_{ij})+2b_{ij}\norm{\bC}_F^2.
	$$
	Setting the partial derivatives to zero finds the result.
	The second part follows similarly.
\end{proof}
This approach leverages the fact that the least squares solution can be found separately for each block, simplifying the calculations and potentially leading to faster convergence. It is especially effective in cases where the matrix $\bA$ displays a structured pattern that can be leveraged through block-wise processing.
The lemma suggests an alternating descent update for obtaining the decomposition which iteratively improve $\bB$ and $\bC$.
The low-rank Kronecker decomposition has a closed-form for each update, and the procedure is outlined  in Algorithm~\ref{alg:ad_Kronecker_zerograd}.

\begin{algorithm}[h] 
\caption{Alternating Descent for Low-Rank Kronecker Decomposition}
\label{alg:ad_Kronecker_zerograd}
\begin{algorithmic}[1] 
\Require Matrix $\bA\in \real^{M\times N}$ with $M=m_1m_2$ and $N=n_1n_2$;
\State Initialize $\bB\in\real^{m_1\times n_1}$ and $\bC\in\real^{m_2\times n_2}$; 
\State Choose a stop criterion on the approximation error $\delta$;
\State Choose the maximal number of iterations $C$;
\State $iter=0$; \Comment{Count for the number of iterations}
\While{$\norm{\bB\otimes \bC -\bA}_F^2>\delta $ and $iter<C$}
\State $iter=iter+1$; 
\State $C_1 = \trace(\bC^\top\bC)$
\For{$i=1,2,\ldots,m_1, j=1,2,\ldots,n_1$}
\State $b_{ij} = \frac{\trace(\bA_{ij}^\top \bC)}{C_1}$;
\EndFor
\State $C_2 = \trace(\bB^\top\bB)$
\For{$i=1,2,\ldots,m_2, j=1,2,\ldots,n_2$}
\State $c_{ij} = \frac{\trace(\widetilde{\bA}_{ij}^\top\bB)}{C_2}$;
\EndFor
\EndWhile
\State Output $\bB$ and $\bC$;
\end{algorithmic} 
\end{algorithm}

\subsection{Missing Entries}
The least squares framework for low-rank Kronecker decomposition can be effectively applied in the Netflix context.
To do this, we introduce an additional mask matrix $\bM\in\{0,1\}^{M\times N}$ where $m_{mn}\in\{0,1\}$ indicates whether the ``user" $n$ has rated the ``movie" $m$ or not. 
As a result, the update for $b_{ij}$ in Equation~\eqref{equation:kro_ls_1} can be computed as  
$$
b_{ij} =\frac{\norm{\bA_{ij} \hadaprod \bC \hadaprod \bM_{ij}}_F^2}{\norm{\bC \hadaprod \bC \hadaprod \bM_{ij}}_F^2}, \gap i\in\{1,2,\ldots, m_1\}, j\in\{1,2,\ldots, n_1\},
$$
where $\bM_{ij}=\bM[(i-1)m_2+1:i\cdot m_2, (j-1)n_2+1:j\cdot n_2]$.
Similarly, the update of $c_{ij}$ in Equation~\eqref{equation:kro_ls_2} can be expressed as:
$$
c_{ij} =\frac{ \norm{\widetilde{\bA}_{ij}\hadaprod\bB\hadaprod \widetildebM_{ij}}_F^2 }{\norm{\bB\hadaprod \bB \hadaprod \widetildebM_{ij}}_F^2}, \gap i\in\{1,2,\ldots, m_2\}, j\in\{1,2,\ldots, n_2\}
$$
where $\widetildebM_{ij}=\bM[i:m_2:M, j:n_2:N]$.

\section{Low-Rank Khatri-Rao Decomposition}\label{section:lrank_khatri_decom}
Since the Khatri-Rao product is closely related to the Kronecker product, it is particularly suited for matrices with column-wise partitioned structures.
In a manner similar to the low-rank Kronecker decomposition, we can also consider the low-rank Khatri-Rao decomposition.
Suppose the design matrix $\bA\in\real^{M\times N}$ has dimensions such that  $M=m_1m_2$. 
Consider two matrices $\bB\in\real^{m_1\times N}$ and $\bC\in\real^{m_2\times N}$ such that their Khatri-Rao product $\bB\odot  \bC\in\real^{M\times N}$ is defined, having the same shape with $\bA$.
Theorem~\ref{theorem:rank_khatri_prod} and \ref{theorem:kkrank_khatri_prod} show that $\rank(\bB\odot  \bC)\geq \max\{\rank(\bB), \rank(\bC)\}$ and $\rank_k(\bB\odot  \bC)\geq \min\{\rank_k(\bB)+\rank_k(\bC)-1, N\}$, indicating the Khatri-Rao product of low-rank matrices $\bB$ and $\bC$ can have more complex structured patterns.
The goal of low-rank Khatri-Rao decomposition then becomes 
\begin{itemize}
\item Given a real matrix $\bA\in \real^{M\times N}$, find  matrix factors $\bB\in\real^{m_1\times N}$ and $\bC\in\real^{m_2\times N}$ such that: 
$$
\min\,\,L(\bB, \bC) = \norm{\bB\odot  \bC-\bA}_F^2.
$$
\end{itemize}

\begin{lemma}[Khatri-Rao Decomposition via Least Squares]\label{lemma:lrank_khatri_opt}
Suppose $\bA\in \real^{M\times N}$ with $M=m_1m_2$. If $\bC\in\real^{m_2\times N}$ is fixed, then the matrix $\bB$ defined by 
\begin{equation}\label{equation:khatri_opt_1}
	b_{ij} =\frac{\bc_j^\top\ba_{ij}}{\bc_j^\top \bc_j}, \gap i\in\{1,2,\ldots, m_1\}, j\in\{1,2,\ldots, N\}
\end{equation}
minimizes $\norm{\bB\odot  \bC-\bA}_F$, where $\ba_{ij}=\bA[(i-1)m_2+1:i\cdot m_2, j]$ is $(i,j)$-th block of $\bA$ with size $m_2\times 1$.
Analogously, if $\bB\in\real^{m_1\times N}$ is fixed, then the matrix $\bC$ defined by 
\begin{equation}\label{equation:khatri_opt_2}
	c_{ij} =\frac{\bb_j^\top\widetilde{\ba}_{ij} }{\bb_j^\top\bb_j}, \gap i\in\{1,2,\ldots, m_2\}, j\in\{1,2,\ldots, N\}
\end{equation}
minimizes $\norm{\bB\odot \bC-\bA}_F$, where $\widetilde{\ba}_{ij}=\bA[i:m_2:M, j]\in\real^{m_1\times 1}$.
\end{lemma}
\begin{proof}[of Lemma~\ref{lemma:lrank_khatri_opt}]
Consider the  loss of the $(i,j)$-th block:
$$
\norm{\ba_{ij} - b_{ij}\bc_j}_2^2 = b_{ij}^2\bc_j^\top\bc_j - 2\bc_j^\top\ba_{ij} b_{ij} + \ba_{ij}^\top\ba_{ij}.
$$
Therefore, $b_{ij} =\frac{\bc_j^\top\ba_{ij}}{\bc_j^\top \bc_j}$ obtains the minimum.
The second part follows similarly.
\end{proof}

The Khatri-Rao decomposition is quite flexible, allowing for the addition of more components in the Khatri-Rao product. This flexibility enables us to extend the basic Khatri-Rao decomposition to incorporate multiple factors, which can be useful in various applications where the data exhibit complex interactions, making it suitable for applications where multiple sources of variation need to be accounted for. The \textit{cascaded Khatri-Rao decomposition} can be formulated as follows:
Given a real matrix $\bA\in\real^{M\times N}$, we aim to find matrix factors $\bB\in\real^{m_1\times N}$, $\bW\in\real^{n\times N}$, $\bC\in\real^{m_2\times N}$, and potentially additional factors, such that: 
$$
\min\,\,L(\bB, \bW, \bC, \ldots) = \norm{\bB\odot  \bW\odot \bC-\bA}_F^2.
$$
\begin{lemma}[Cascaded Khatri-Rao Decomposition via Least Squares]\label{lemma:casc_lrank_khatri_opt}
Suppose $\bA\in \real^{M\times N}$ with $M=m_1\cdot\textcolor{mylightbluetext}{n}\cdot m_2$. If $\textcolor{mylightbluetext}{\bW\in\real^{n\times N}}$ and $\bC\in\real^{m_2\times N}$ are fixed, then the matrix $\bB$ defined by 
\begin{equation}\label{equation:cas_khatri_opt_1}
b_{ij} =\frac{\textcolor{mylightbluetext}{\widetildebc_j}^\top\ba_{ij}}{\textcolor{mylightbluetext}{\widetildebc_j}^\top \textcolor{mylightbluetext}{\widetildebc_j}}, \gap i\in\{1,2,\ldots, m_1\}, j\in\{1,2,\ldots, N\}
\end{equation}
minimizes $\norm{\bB\odot \bW\odot  \bC-\bA}_F$, where $\ba_{ij}=\bA[(i-1)\textcolor{mylightbluetext}{\cdot (n\cdot m_2)}+1:i \textcolor{mylightbluetext}{\cdot (n\cdot m_2)}, j]$ is $(i,j)$-th block of $\bA$ with size $\textcolor{mylightbluetext}{ (n\cdot m_2)}\times 1$, and $\textcolor{mylightbluetext}{\widetildebc_j}$ is the $j$-th column of $\bW\odot \bC$.

Analogously, if $\textcolor{mylightbluetext}{\bW\in\real^{n\times N}}$ and $\bB\in\real^{m_1\times N}$ are fixed, then the matrix $\bC$ defined by 
\begin{equation}\label{equation:cas_khatri_opt_2}
c_{ij} =\frac{\textcolor{mylightbluetext}{\widetildebb_j}^\top\widetilde{\ba}_{ij} }{\textcolor{mylightbluetext}{\widetildebb_j}^\top\textcolor{mylightbluetext}{\widetildebb_j}}, \gap i\in\{1,2,\ldots, m_2\}, j\in\{1,2,\ldots, N\}
\end{equation}
minimizes $\norm{\bB\odot\bW\odot \bC-\bA}_F$, where $\widetilde{\ba}_{ij}=\bA[i:\textcolor{black}{ m_2}:M, j]\in\real^{\textcolor{mylightbluetext}{n\cdot m_1}\times 1}$, and $\textcolor{mylightbluetext}{\widetildebb_j}$ is the $j$-th column of $\bB\odot \bC$.
\paragraph{Obtain $\bW$. }
Up to this point, the result is the same as Lemma~\ref{lemma:lrank_khatri_opt} with careful consideration given to the dimensions.
If $\bB\in\real^{m_1\times N}$ and $\bC\in\real^{m_2\times N}$ are fixed, then the matrix $\bW$ defined by 
\begin{equation}\label{equation:cas_khatri_opt_3}
w_{ij} 	=\frac{\widetildebw_{ij}^\top\widehat{\ba}_{ij} }{\widetildebw_{ij}^\top\widetildebw_{ij}}, \gap i\in\{1,2,\ldots, n\}, j\in\{1,2,\ldots, N\}
\end{equation}	
minimizes $\norm{\bB\odot\bW\odot \bC-\bA}_F$, where 
$\widehat{\ba}_{ij} =\bA[I_s, j] \in\real^{m_1\cdot m_2\times 1}$, and 
$$
\widetildebw_{ij}=
[b_{1j}  \bc_j, b_{2j}  \bc_j, \dots b_{m_1,j}  \bc_j]^\top \in\real^{m_1\cdot m_2\times 1}.
$$
The $I_s$ is an index set that contains $m_1$ blocks with each block containing $m_2$ components:
$$
\begin{aligned}
&\gap I_s \\
&= \{(j:j+m_2-1), (j+n:j+n+m_2-1), \ldots  (j+(m_1-1)n:j+(m_1-1)n+m_2-1)\}, 
\end{aligned}
$$
\end{lemma}
\begin{proof}[of Lemma~\ref{lemma:casc_lrank_khatri_opt}]
The $j$-th column of $\bB\odot\bW\odot \bC$ is 
$$
\bb_i \otimes \bw_i \otimes \bc_i 
= 
[
(b_{1j} w_{1j} \bc_j, b_{1j} w_{2j} \bc_j, \dots b_{1j} w_{nj} \bc_j)\mid 
(b_{2j} w_{1j} \bc_j, b_{2j} w_{2j} \bc_j, \dots b_{2j} w_{nj} \bc_j)\mid
\ldots 
]^\top.
$$
The result follows by solving  least squares problem with slicing each $w_{ij}$ for $i\in\{1,2,\ldots, n\}, j\in\{1,2,\ldots, N\}$.
\end{proof}

With this proof at hand, it is then straightforward to extend the cascaded Khatri-Rao decomposition to more than three matrices: $\bA=\bA_1\odot \bA_2\odot \bA_3\odot \bA_4\odot\ldots$.
The update of $\bA_2$ can be obtained by letting $\bA_1:=\bB$,  $\bA_2:=\bW$, $\bA_3\odot \bA_4\odot\ldots:=\bC$ in Lemma~\ref{lemma:casc_lrank_khatri_opt}; 
and similarly,  the update of $\bA_3$ can be obtained by letting $\bA_1\odot \bA_2:=\bB$,  $\bA_3:=\bW$, $ \bA_4\odot \ldots:=\bC$ in Lemma~\ref{lemma:casc_lrank_khatri_opt}.

\section{Low-Rank Adaptation in Large Language Models}\label{sec:lora_in_llm}

Low-rank adaptation (LoRA) and its variants are techniques designed to adapt large pre-trained language models for specific tasks or domains with minimal changes to the original model's parameters.
Instead of modifying the pre-trained model weights, LoRA works by freezing the pre-trained model weights and introducing trainable rank decomposition matrices into each layer of the transformer architecture. 
This approach dramatically  reduces the number of trainable parameters and the GPU memory requirement, making it significantly more efficient than full fine-tuning.
On the other hand, LoRA and its variant  do not introduce any additional latency during inference, as the trainable matrices can be merged with the frozen weights after fine-tuning, eliminating any extra computations during inference.

Full fine-tuning, which requires retraining all the parameters of a pre-trained model, becomes increasingly impractical for most small and medium-sized companies as the size of these models grows. 
For example, \textit{GPT-3 175B} has 175 billion parameters \citep{brown2020language},  making it extremely costly in terms of computational resources and storage to deploy multiple instances of such a model, each fine-tuned for a different downstream task.
Additionally, storing a complete model checkpoint for each downstream application is inefficient for deployment and task-switching, especially with large models.

LoRA addresses these challenges by significantly reducing the number of trainable parameters, which in turn reduces the computational and memory overhead during both training and inference. For instance, compared to GPT-3 175B  fine-tuned with Adam, LoRA can reduce the number of trainable parameters by 10,000 times and the GPU memory requirement by 3 times \citep{hu2021lora}.

LoRA also enhances task-switching flexibility by allowing a pre-trained model to be shared and used for multiple tasks. 
This is achieved by freezing the shared model and efficiently switching tasks by replacing the small LoRA matrices (the matrices $\bA$ and $\bB$ in Figure~\ref{fig:lora_loha_lokr}), thereby reducing both storage requirements and task-switching overhead.

\begin{figure}[h]
	\centering
	\includegraphics[width=1.0\textwidth]{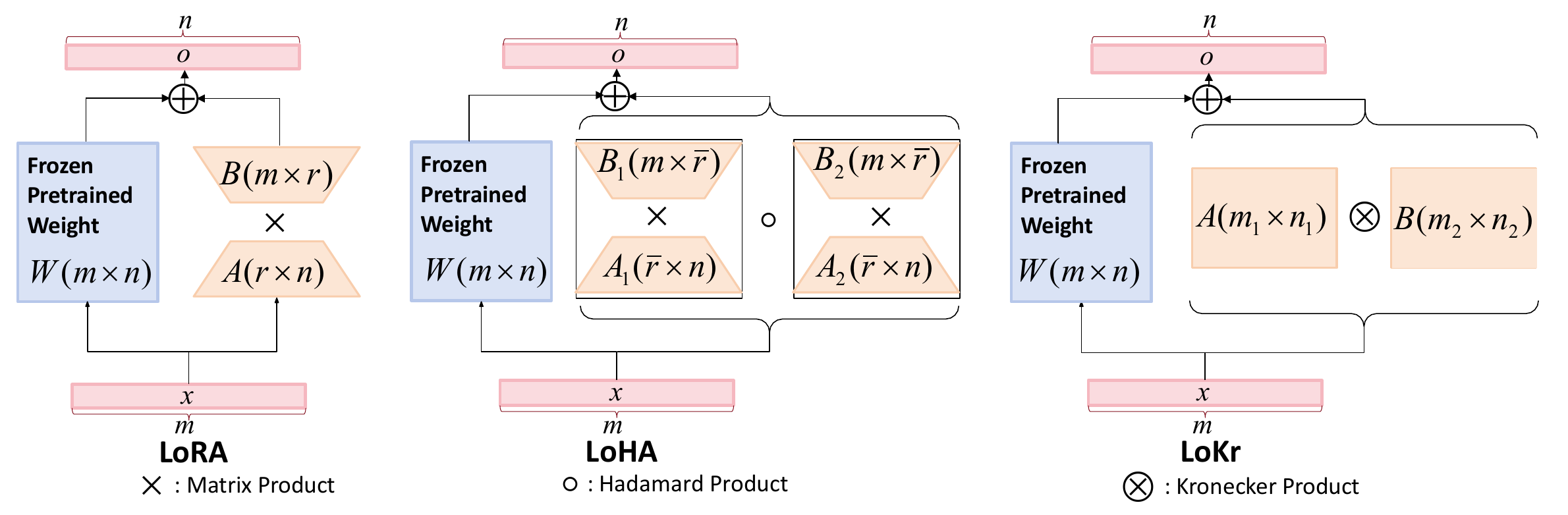}
	\caption{Diagram illustrating LoRA, LoHA, and LoKr (KronA).}
	\label{fig:lora_loha_lokr}
\end{figure}
\paragraph{LoRA.}  To be more specific, the weight update $\Delta \bW \in \real^{m\times n}$ is decomposed into two low-rank matrices $\bB \in\real^{m\times r}$ and $\bA\in\real^{r\times n}$, with $m$ and $n$ being the dimensions of the original model parameters, and $r$ being the rank of the decomposition, satisfying $r \leq  \min(m,n)$. 
During the fine-tuning phase, the pre-trained model parameter $\bW$ remains unchanged (usually called frozen) while only the low-rank matrices $\bB$ and $\bA$ are updated.
The forward pass, originally defined as $\bo = \bW \bx + \bb$, is adjusted to:
\begin{equation}\label{equation:lora_eq}
	\bo = \bW\bx+\bb +\alpha \Delta \bW\bx
	=\bW\bx+\bb +\alpha \bB\bA\bx,
\end{equation}
where $\bb$ is the bias term, and $\alpha$ is called the \textit{merging ratio}, which controls the balance between preserving the pre-trained model's information and adapting it to  new target concepts. Hence the name \textit{low-rank adaptation (LoRA)} \citep{hu2021lora}.

\paragraph{LoHA.} 
In particular, it is widely recognized that methods based on matrix factorization are limited by the low-rank constraint. 
Within the LoRA framework, weight updates are restricted to the low-rank space, which can affect the performance of the fine-tuned model. 
We hypothesize that achieving better fine-tuning performance may require a relatively higher rank, particularly when working with larger fine-tuning datasets or when there is a significant difference between the data distribution of downstream tasks and the pretraining data. 
However, this could lead to increased memory usage and greater storage requirements.
We have demonstrated in Theorem~\ref{theorem:rank_hada_prod} that the Hadamard product of two rank-$\overline{r}$ matrices $\bW_1=\bB_1\bA_1$ and $\bW_2=\bB_2\bA_3$ has a rank of at most $\overline{r}^2$.
Given that  $\bA$ and $\bB$ in Equation~\eqref{equation:lora_eq} require $(m+n)r$ floating-points numbers, if we set $\overline{r} = \frac{r}{2}$, the number of floating-point numbers remains $(m+n)r$ for the Hadamard product of $(\bB_1\bA_1)\hadaprod (\bB_2\bA_3)$, while the rank of the low-rank approximation becomes $\frac{r^2}{4}$ at most. 
This technique is frequently referred to as  \textit{LoHA (low-rank adaptation with Hadamard product)} in the field (originally proposed to address the low-rank constraint issue in federated learning problems  \citep{hyeon2021fedpara}). 
When $r>4$, LoHA can represent more complex models compared to LoRA.
The forward pass then becomes:
$$
\bo = \bW\bx+\bb +\alpha \Delta \bW\bx
=\bW\bx+\bb +\alpha \left((\bB_1\bA_1)\hadaprod (\bB_2\bA_2)\right) \bx.
$$
\paragraph{Warranty.} We have demonstrated in Section~\ref{section:low_rank_hadamard} that not all matrices can be decomposed using the Hadamard product of low-rank matrices. 
Consequently, when $\overline{r}=\frac{r}{2}$,  LoRA can explore all low-rank matrices within a space of rank $r$, whereas the LoHA method searches within an unspecified subset of the space with a rank up to $\frac{r^2}{4}$. 
The specific advantages of LoHA for federated learning or fine-tuning pre-trained LLMs have yet to be fully clarified.

\paragraph{LoKr.} 
Similarly, the low-rank adaptation can be represented by a Kronecker product of two low-rank matrices, $\bA\in\real^{m_1\times n_1}$ and $\bB\in\real^{m_2\times n_2}$, where $m=m_1m_2$ and $n=n_1n_2$. 
This approach is commonly known as \textit{LoKr (low-rank adaptation with Kronecker product, a.k.a., Kronecker adapter)} in the literature \citep{edalati2022krona, yeh2023navigating}.
The forward pass is then given by 
$$
\bo = \bW\bx+\bb +\alpha \Delta \bW\bx
=\bW\bx+\bb +\alpha \left(\bA\otimes \bB\right) \bx.
$$
Alternatively, one of the factors can be further  approximated by a low-rank matrix decomposition. For instance, consider  $\bB=\bC\bD$ with $\bC\in\real^{m_2\times k}$ and $\bD\in\real^{k\times n_2}$. In this case, the LoKr can be further represented as 
$$
\bo = \bW\bx+\bb +\alpha \Delta \bW\bx
=\bW\bx+\bb +\alpha \left(\bA\otimes (\bC\bD)\right) \bx.
$$
A comparison of the three adaptation methods is illustrated in Figure~\ref{fig:lora_loha_lokr}.

\begin{figure}[h]
	\centering
	\includegraphics[width=1.02\textwidth]{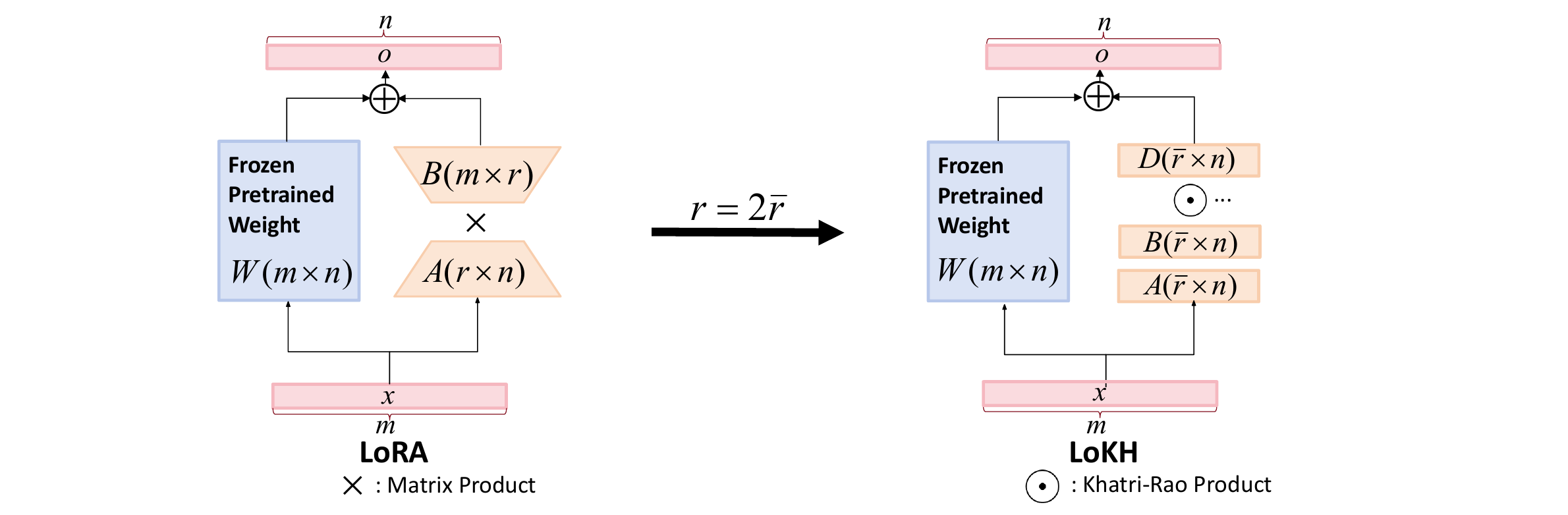}
	\caption{Diagram illustrating LoRA and LoKH.}
	\label{fig:lora_loKH}
\end{figure}

\paragraph{Insights from Khatri-Rao Products: LoKH.}
We consider a special case where $\bW\in\real^{n\times n}$ and each low-rank matrix has the same rank and $k$-rank. This is not uncommon since the matrix structures in language models are extremely complex and usually have full rank. In the context of transformer architectures, the $n$ is typically a power of $2$. As a result, there exists an integer $\overline{r}$ such that $\overline{r}^4 = n$. Let $\bA,\bB, \bC, \bD\in\real^{\overline{r}\times n}$ represent the low-rank matrices, the \textit{low-rank adaptation with Khatri-Rao product (LoKH)} can be expressed as (see Figure~\ref{fig:lora_loKH}):
$$
\bo = \bW\bx+\bb +\alpha \Delta \bW\bx
=\bW\bx+\bb +\alpha \left(\bA\odot  \bB \odot \bC\odot \bD\right) \bx.
$$
In this context, the LoRA with rank $r$ can represent models with a rank of at most $r$. However, if we explore the space where the rank is equal to the  $k$-rank, the LoKH method can identify models with a rank greater than  $4\overline{r}-3$ (Theorem~\ref{theorem:kkrank_khatri_prod}).
When $r>3$, LoKH can represent more complex models compared to LoRA. 
Therefore, LoKH (rank $2r-3$) offers a balance between LoRA (rank $r$) and LoHA (rank $\frac{r^2}{4}$).
LoHA represents a wider range of models due to its higher rank, while in a small subset in this high-rank space; LoRA offers a simpler model.
To the best of our knowledge, the LoKH method has not been explored in the fields of text-to-image models or large language models. This presents an opportunity for future research to investigate the potential benefits of LoKH in these domains, potentially leading to improved performance and efficiency in transformer-based models.
Furthermore, the LoKH method is straightforward to extend by cascading more components, allowing for the representation of more complex models: $\Delta\bW= \bA\odot \bB\odot \bC\odot \bD\odot \bE\ldots$.

Additionally, the Khatri-Rao decomposition has the advantage of aligning the columns of the weight matrix $\bW$. In the transformer structure, the attention mechanism is calculated as
$$
\text{Attention} (\bQ, \bK, \bV) = \text{softmax}\left( \frac{\bQ^\top\bK}{\sqrt{d_k}}\right)\bV^\top,
$$
where $\bQ, \bK \in\real^{d_k\times n}$ are the query and key matrices, $\bV\in\real^{ d_v\times n}$ is the value matrix, $n$ is the length of sequence (tokens for texts and patches for images), and $d_k, d_v$ are lengths of dimensions. Specifically, the attention  captures the interactions between the query, key, and value matrices in a way that emphasizes the matching of features across different tokens/patches.
The LoKH method on these matrices, on the other hand, can be seen as a technique that facilitates the matching of features for each token/patch in the sequence before sending them into the attention mechanism.
By leveraging the Khatri-Rao product, LoKH enables the model to focus on the specific features of individual tokens/patches instead of interleaving the features of different tokens/patches (before sending them into the attention mechanism), which can be beneficial for tasks requiring fine-grained attention mechanisms.

Finally,the Khatri-Rao method offers flexibility through partition-wise decomposition (Definition~\ref{definition:partition_khatri_prod}). In this case, the low-rank decomposition divides the tokens/patches into several blocks, facilitating more extensive local information exchange within each block.
The method can be made even more flexible to identify various shifted windows across the blocks and then combine two (or more) partition-wise decompositions using the Hadamard product. 
The low-rank decomposition then assumes local/group connections for some tokens/patches, while each group is not connected (akin to the group Lasso formulation \citep{yuan2006model, bach2011convex}).
This approach can be tailored to transformer architectures for computer vision, akin to the shifted windows and patch merging techniques used in the Swin Transformer \citep{liu2021swin}.

\begin{figure}[h]
	\centering
	\includegraphics[width=1.0\textwidth]{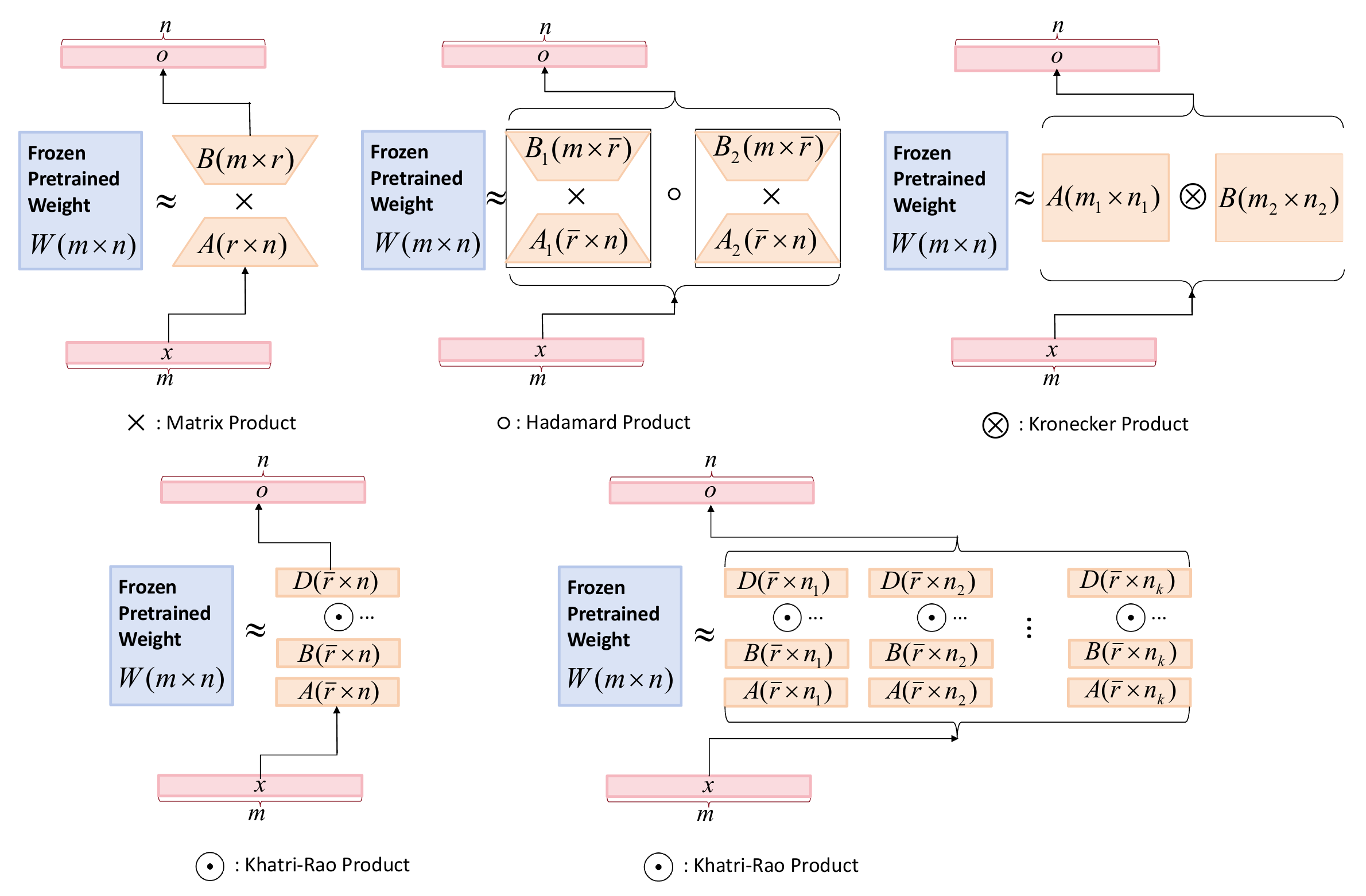}
	\caption{Diagram illustrating low-rank approximation for transformer architectures. The bottom-right figure illustrates the cascading of Khatri-Rao products such that $n=n_1n_2\ldots n_k$.}
	\label{fig:lora_loha_lokr_approx}
\end{figure}

\subsection{Low-Rank Approximation of Transformers}

Low-rank matrix decomposition methods have been successfully used to compress deep neural networks (see, for example,  Chapter 15 of \citet{lu2021numerical}). These methods aim to reduce the number of parameters in the network by approximating the weight matrices of the layers using lower rank matrices, and the network is subsequently fine-tuned. 
This reduction in complexity makes the models more suitable for resource-constrained environments, such as mobile or embedded systems.

The key idea behind these methods is that the weight matrices of neural network layers can often be well-approximated by the product of two (or more) smaller matrices. For example, consider a fully connected layer with a weight matrix $\bW\in\real^{m\times n}$. Instead of storing $\bW$ directly, we can approximate it as: $\bW\approx\bA\bB$, where $\bA\in\real^{m\times r}$ and $\bB\in\real^{r\times n}$, and $r$ is much smaller than both $m$ and $n$. This results in a significant reduction in the number of parameters needed to represent the layer, from $mn$ to  $mr+nr$.

Given  the success of methods like LoRA, LoHA, and LoKr in the field of fine-tuning pre-trained large language models, it is valuable to explore the application of low-rank decomposition methods on pre-trained transformer architectures, which have a large number of parameters due to their self-attention mechanisms and feed-forward networks. By applying low-rank approximations (e.g., ALS, low-rank Hadamard decomposition, low-rank Kronecker decomposition, and low-rank Khatri-Rao decomposition; see Figure~\ref{fig:lora_loha_lokr_approx}) to compress these transformers and then fine-tuning the compressed models for downstream tasks, we can potentially achieve substantial reductions in model size without compromising performance.
Unlike low-rank adaptation methods, the low-rank representation of transformers not only reduces the memory usage for the forward pass (during both training and inference) but also reduces those in the backward propagation, thereby enhancing the efficiency of the entire training and inference processes.

\setcitestyle{numbers}
\small
\bibliography{bib}
\bibliographystyle{iclr}

\end{document}